\documentclass[10pt,twocolumn,letterpaper]{article}
\usepackage{cvpr}              
\usepackage{amsthm}
\usepackage{multirow}
\usepackage{algorithm}
\usepackage{algorithmic}
\usepackage[table]{xcolor}
\newtheorem{theorem}{Theorem}
\newtheorem{proposition}{Proposition}

\definecolor{cvprblue}{rgb}{0.21,0.49,0.74}
\definecolor{lightorange}{RGB}{230, 126, 34} 
\definecolor{lightgreen}{RGB}{46,204,113} 
\definecolor{lightblue}{RGB}{52, 152, 219} 
\usepackage[pagebackref,breaklinks,colorlinks,allcolors=cvprblue]{hyperref}

\renewcommand{\thefootnote}{\fnsymbol{footnote}}
\definecolor{cvprblue}{rgb}{0.21,0.49,0.74}
\usepackage[pagebackref,breaklinks,colorlinks,allcolors=cvprblue]{hyperref}
\usepackage[accsupp]{axessibility} 
\def\paperID{} 
\def\confName{CVPR}
\def\confYear{2026}

\title{Saliency-Guided Representation with Consistency Policy Learning for Visual Unsupervised Reinforcement Learning}

\author{
Jingbo Sun\textsuperscript{1,2,3}, 
Qichao Zhang\textsuperscript{1,3*}, 
Songjun Tu\textsuperscript{1,2,3}, 
Xing Fang\textsuperscript{1,3}, 
Yupeng Zheng\textsuperscript{1,3}, 
Haoran Li\textsuperscript{1,3}, \\
Ke Chen\textsuperscript{2*}, 
Dongbin Zhao\textsuperscript{1,2,3*} \\
{\small
\textsuperscript{1}SKL-MAIS, Institute of Automation, Chinese Academy of Sciences,
\textsuperscript{2}Pengcheng Laboratory,
}
{\small
\textsuperscript{3}School of Artificial Intelligence, University of Chinese Academy of Sciences
}
}

\begin{document}
\maketitle
\let\thefootnote\relax\footnotetext{%
\hspace{-2.5em} 
* Corresponding authors. This work is supported by National Key Research and Development Program of China under Grant 2022YFA1004000, Beijing Natural Science Foundation-Xiaomi Innovation Joint Fund L253007, Beijing Natural Science Foundation under Grant 4242052. Our code is available at \url{https://github.com/bofusun/SRCP}.
}
\begin{abstract}
Zero-shot unsupervised reinforcement learning (URL) offers a promising direction for building generalist agents capable of generalizing to unseen tasks without additional supervision.
Among existing approaches, successor representations (SR) have emerged as a prominent paradigm due to their effectiveness in structured, low-dimensional settings.
However, SR methods struggle to scale to high-dimensional visual environments.
Through empirical analysis, we identify two key limitations of SR in visual URL: (1) SR objectives often lead to suboptimal representations that attend to dynamics-irrelevant regions, resulting in inaccurate successor measures; and (2) these flawed representations hinder SR policies from modeling multi-modal skills and ensuring skill controllability.
To address these limitations, we propose \textbf{S}aliency-Guided \textbf{R}epresentation with \textbf{C}onsistency \textbf{P}olicy Learning (SRCP), a novel framework that enhances zero-shot generalization of SR methods in visual URL. 
SRCP decouples representation learning from successor training by introducing a saliency-guided dynamics task to capture dynamics-relevant representations, thereby improving successor measure.
Moreover, it integrates a fast-sampling consistency policy 
with URL-specific classifier-free guidance and tailored training objectives to improve skill-conditioned policy modeling and controllability.
Extensive experiments demonstrate that SRCP achieves state-of-the-art zero-shot generalization in visual URL and and remains compatible with various SR methods.
\end{abstract}    
\vspace{-.6cm}
\section{Introduction}
\label{sec:intro}

\begin{figure}[t]
\centering
\includegraphics[width=0.95\columnwidth]{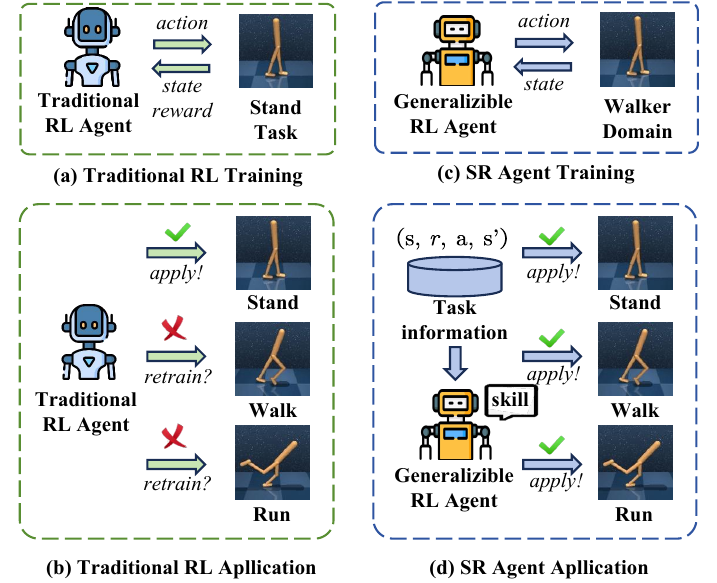} 
\caption{
Illustration of traditional RL and SR methods:
(a) Traditional RL trains with task-specific rewards.
(b) Traditional RL generalizes poorly to tasks with different rewards.
(c) SR methods learn skill-conditioned skills without reward.
(d) SR agents infer skills from minimal task information and generalize across tasks.
}
\label{figr0}
\vspace{-.5cm}
\end{figure}

Reinforcement learning (RL)~\cite{sutton_rl,lu2025equilibrium,PCM} has achieved remarkable success in various domains, including games~\cite{silver2017mastering,liu2025videos}, autonomous driving~\cite{zhang2022trajgen,lxy}, and robotics~\cite{intelligence2025pi,sun2025salience,zhao2025learning}.
Despite successes, most RL methods rely on task-specific reward functions, which limits their ability to generalize to tasks with different rewards or objectives.
Given the varied demands of real-world applications, it is crucial to develop agents capable of generalizing to unseen tasks without additional training (zero-shot generalization) or adapting via finetuning.
Unsupervised reinforcement learning (URL)~\cite{urlb,comsd} offers a promising solution by pretraining agents on reward-free data to acquire generalizable policies.

\begin{figure*}[t]
\centering
\includegraphics[width=2.1\columnwidth]{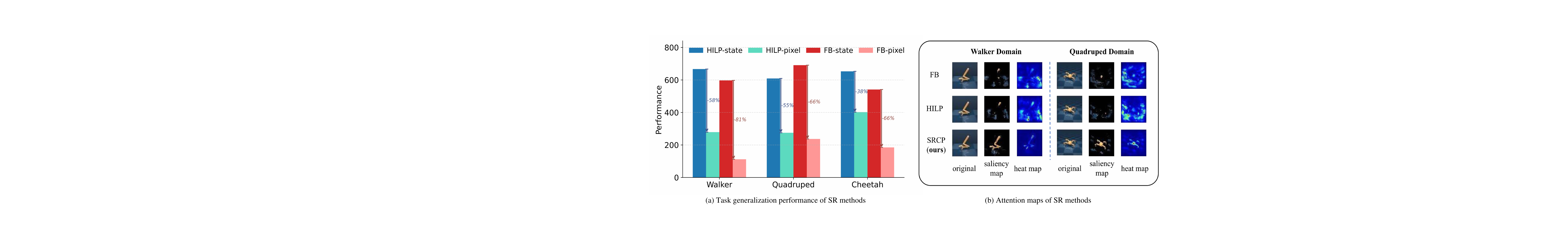} 
\caption{Generalization performance and attention analysis of prior methods and SCPL. (a) Task generalization performance of previous SR methods (HILP and FB) under low-dimensional state and high-dimensional visual inputs. (b) Comparison of saliency mask maps and attention heatmaps of prior methods and SRCP in the DMC Walker and Quadruped domains.}
\label{figr1}
\vspace{-.4cm}
\end{figure*}

Among various URL approaches, successor representation (SR) methods~\cite{sf,fb1} stand out for their zero-shot generalization capability by decoupling reward learning from environment dynamics via successor feature factorization (Fig.\ref{figr0}).
Despite their theoretical advantages, SR methods face significant challenges in high-dimensional visual environments, which are common in real-world RL applications.
As illustrated in Fig.\ref{figr1}(a), while SR methods, such as FB~\cite{fb} and HILP~\cite{hilp}, exhibit strong generalization in low-dimensional state-based URL settings, their performance drops sharply in visual URL tasks.
This discrepancy raises a key question: \textit{Why do SR methods, which excel in state-based URL, struggle in high-dimensional visual URL?}

We identify a critical limitation in extending SR methods to visual URL: it is difficult to learn effective representations from visual inputs using only the SR objective. 
As shown in Fig.\ref{figr1}(b), the learned representations of SR methods tend to focus on task-irrelevant regions of the observation. 
Our empirical analysis reveals that \textbf{such suboptimal representations lead to inaccurate estimation of successor measures, thereby degrading generalization performance}.
This representation challenge further complicates policy learning.
Unlike in state-based URL where policies are learned from structured, low-dimensional inputs, visual URL requires learning skill-conditioned policies from visual observations.
Our experiments show that this setting \textbf{poses challenges for SR policies in modeling multi-modal skill-conditioned action distributions while maintaining skill controllability.}
Traditional policy networks often struggle with this trade-off, whereas diffusion models offer expressive modeling at the cost of high inference latency.
In contrast, consistency models emerge as a lightweight yet effective alternative, enabling expressive behaviors with improved skill controllability in visual URL.

Building on these insights, we propose \textbf{S}aliency-guided \textbf{R}epresentation and \textbf{C}onsistency \textbf{P}olicy learning (SRCP), a unified framework that enhances zero-shot task generalization in visual URL.
SRCP introduces a saliency-guided dynamic representation learning that decouples dynamics-relevant representations from the SR objective, enabling the encoder to focus on task-relevant regions.
Moreover, SRCP employs a consistency policy with classifier-free guidance and tailored objectives to model multi-modal skill-conditioned policies with improved skill controllability.
Moreover, SRCP is compatible with various SR methods, making it a versatile and extensible framework.
To the best of our knowledge, SRCP is the first framework that explicitly targets zero-shot generalization in visual URL.
Our main contributions are summarized as follows:
\begin{itemize}
\item We investigate the failure of SR methods in visual URL and find that representation learning often yields suboptimal features, which hinder both the estimation of successor measures and the modeling of expressive skill-conditioned policies, ultimately limiting generalization.
\item We propose SRCP, a unified framework that decouples representation learning from the SR objective via saliency-guided dynamics-aware learning, enabling the extraction of dynamics-relevant features.
\item We design a consistency policy with specific classifier-free guidance to capture multi-modal skill-conditioned action distributions with enhanced skill controllability.
\item We validate SRCP on 16 challenging visual control tasks across 4 datasets in the ExORL benchmark, demonstrating superior zero-shot generalization performance.
\end{itemize}

\vspace{-.15cm}
\section{Related Work}
\label{sec:formatting}
\vspace{-.1cm}
\subsection{Task Generalization in URL}
Unsupervised reinforcement learning~\cite{urlb, rnd, disagreement, apt, smm} aims to pretrain policies on task-agnostic datasets and generalize them to downstream tasks through task adaptation or zero-shot generalization.
URL methods can be broadly classified into unsupervised skill discovery (USD) approaches and successor representation (SR) methods.
USD methods~\cite{diayn, dads, cic, lsd, metra, cesd} learn diverse skills by maximizing the divergence between skill-specific and average state distributions.
However, the lack of alignment between skills and task objectives limits their zero-shot generalization~\cite{dvfb, geometry1, geometry2}.
SR methods~\cite{usf, sf} enable zero-shot generalization by associating skills with rewards. 
They can be categorized into successor feature (SF)~\cite{usf, sf} and forward-backward representation (FB)~\cite{fb, fb1, vcfb}. 
SF methods learn basic feature via various representations, including contrastive learning~\cite{cl}, Laplacian eigenfunctions~\cite{lap}, low-rank approximations~\cite{lra-p} and Hilbert representations~\cite{hilp}.
In contrast, FB methods~\cite{fb1} estimate successor measures without explicit feature learning.
SR methods can infer near-optimal skills from limited data, enabling zero-shot generalization.
While SR methods generalize well in state-based URL, they demonstrate limited performance in visual domains.
To address this limitation, we propose the SRCP framework to improve generalization of SR methods in visual URL.

\vspace{-.1cm}
\subsection{Diffusion Actor in RL}
\vspace{-.1cm}
Diffusion models~\cite{diffusion} have demonstrated remarkable capabilities in modeling complex distributions and generating diverse, high-quality samples. 
These properties make them particularly suitable for RL, where they are commonly adopted to model diverse policies~\cite{imit_diff, offline_diff, edp}, supporting expressive action distributions.
However, the practical deployment of diffusion models in RL often suffers from high inference cost. To address this, recent works have explored consistency models as lightweight alternatives~\cite{cm2024, cccp,ligeneralizing}, replacing diffusion models within actor-critic frameworks while maintaining competitive performance in continuous control tasks.
Despite the success of diffusion and consistency models in RL, their applicability to visual URL with skill-conditioned policies remains underexplored.
In this work, we propose a consistency-based policy with classifier-free guidance to promote multi-modal action distributions and enhance skill controllability in visual URL.
\section{Preliminaries}

\vspace{-.1cm}
\subsection{Zero-shot Task Generalization in URL}
Zero-shot task generalization refers to an agent’s ability to directly applied to new tasks with different reward functions, without additional training~\cite{fb}. 
In URL, the agent is pre-trained with reward-free data to learn skill-conditioned policies. 
For zero-shot generalization, the agent infers skills for a downstream task from a small dataset containing reward signals. 
The inferred skills are then used to produce a policy for the downstream task without  further learning.

\vspace{-.1cm}
\subsection{Successor Representations}
Given a basic feature map $\varphi: \mathcal{S} \to \mathbb{R}^d$, the successor feature represents the expected discounted sum of future features under policies $\pi_z$ indexed by $z \in \mathbb{R}^d$:
$\psi(s_0, a_0, z) = \mathbb{E} \left[ \sum_{t=0}^{\infty} \gamma^t \varphi(s_{t+1}) \,\middle|\, (s_0, a_0), \pi_z \right]$.
Given a reward function $r(s)$, the corresponding $Q$-function under policy $\pi$ can be expressed as: $Q_r^\pi(s_0, a_0) = M^\pi r = \psi^\pi(s_0, a_0)^\top z,$ where the reward function is linearly approximated as $r(s) \approx \varphi(s)^\top z$.
After pretraining, given a reward function $r(s)$, the optimal skill $z_r$ can be computed with: $z_r := \mathbb{E}_\rho[\varphi \varphi^\top]^{-1} \, \mathbb{E}_\rho[\varphi r],$
which enables zero-shot generalization via skill inference.

\subsection{SR Objective in Visual URL}
In visual URL, an encoder is introduced to extract visual features, and both the encoder and successor modules are jointly optimized under the SR objective.



\vspace{-.1cm}
\section{What Limits SR methods in Visual URL?}
\vspace{-.1cm}
\paragraph{How representations limit generalization in visual SR?}

\begin{figure}[t]
\centering
\includegraphics[width=1.0\columnwidth]{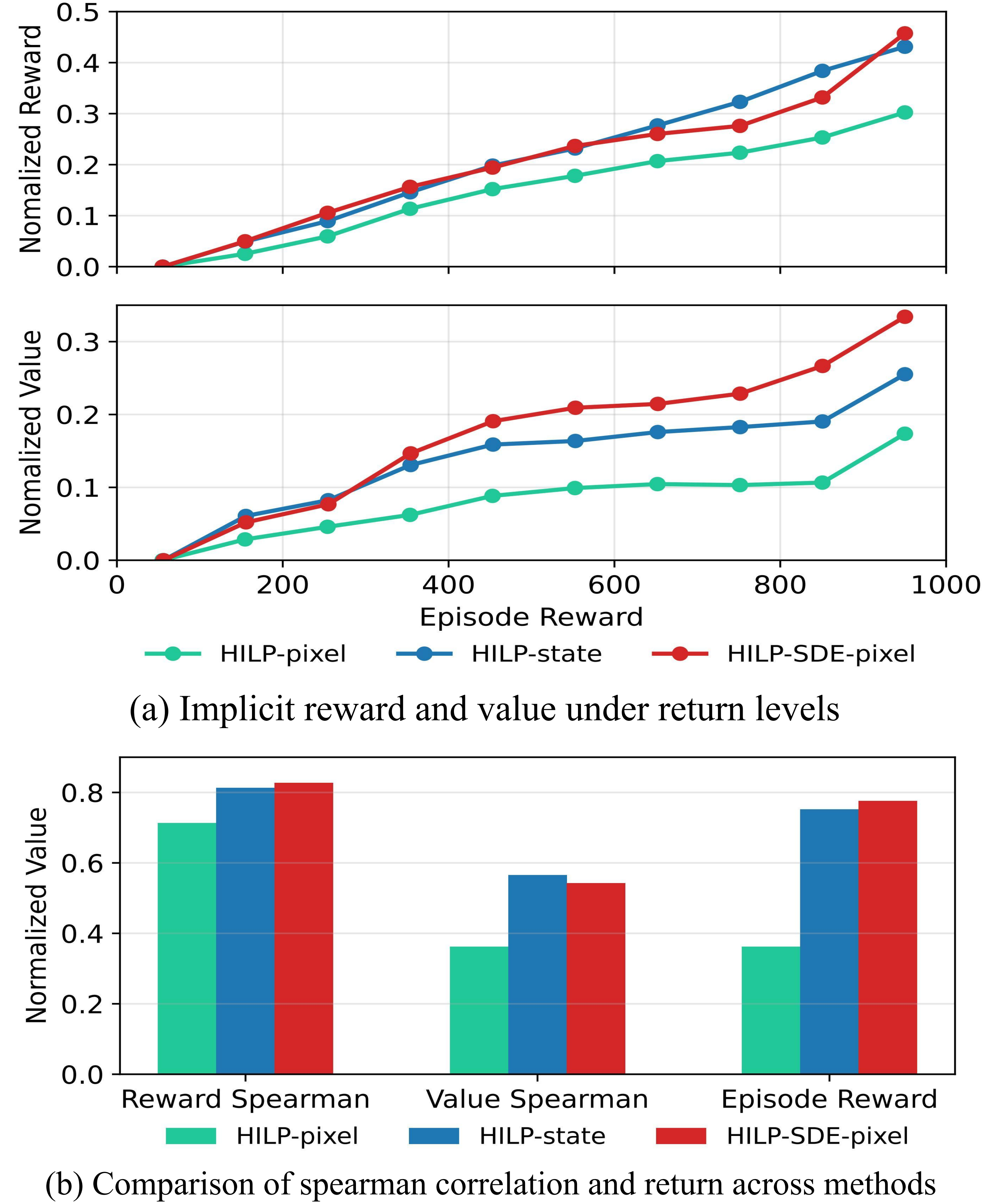} 
\caption{
Value and performance analysis of methods in Walker Stand task. (a) Implicit reward and value estimations of methods across trajectories with varying returns; (b) Spearman correlations of reward and value with return, and performance across methods.
}
\label{figr3}
\vspace{-.3cm}
\end{figure}
Unlike state-based SR, which trains only the successor networks under the SR objective, visual SR jointly optimizes both the encoder and successor networks. 
As shown in Fig.\ref{figr1}, this coupling often biases the learned representations toward dynamics-irrelevant regions. 
Since these representations are further used to derive both basic and successor features for reward and value estimation, which is essential for zero-shot generalization.
We hypothesize that the tight coupling training leads to suboptimal representations,  impairing both the basic features and successor measures, and ultimately limiting generalization.
To validate this hypothesis, we compare three HILP variants on the \textit{Walker Stand} task: HILP-pixel (visual input), HILP-state (state input), and HILP-SDE-pixel (visual input with our saliency-guided encoder).
Fig.\ref{figr3} (a) shows that all three methods produce reward estimates increasing with trajectory returns, indicating robust basic feature learning.
However, value estimates from HILP-state and HILP-SDE-pixel closely track trajectory returns, while HILP-pixel exhibits a weaker value-return trend due to suboptimal representations. 
Fig.\ref{figr3} (b) compares Spearman correlations between trajectory returns and predicted rewards or values, along with average cumulative returns. 
All methods achieve high reward-return correlations, but only HILP-state and HILP-SDE-pixel maintain strong value-return correlations and high cumulative returns. 
In contrast, HILP-pixel shows a weaker correlation and reduced task returns, suggesting inferior value estimation quality.
These results demonstrate that the entangled optimization of the encoder and SR training leads to suboptimal representations, which \textbf{primarily impair successor measures while having minimal impact on basic features, thereby limiting generalization in visual URL.}

\vspace{-.3cm}
\paragraph{Theoretical analysis of successor feature approximation.}
The experimental results above are grounded in the theory of successor features. Let $\psi^{\pi_z}$ denote the true successor features of policy $\pi_z$, and $\hat\psi^z$ be the learned approximation. The generalization performance of the policy $\pi_{z_r}$, measured by its deviation from the optimal value function, is bounded by the approximation error of the successor features:
\begin{equation}
\|\hat V^{\pi_{z_r}} - V^\star\|_\infty
\le
\frac{3\|z_r\|_*}{1-\gamma}\;
\sup_{s,a}\|\ \epsilon  \|,
\end{equation}
where \( \epsilon=\hat\psi^{z_r}(s,a)-\psi^{\pi_{z_r}}(s,a)\). The result highlights that the quality of the successor features impacts the generalization. Detail proofs are provided in Appendix C.


\vspace{-.3cm}
\paragraph{Do visual URL tasks demand more effective skill-conditioned policies?}
\begin{figure}[t]
\centering
\includegraphics[width=1.0\columnwidth]{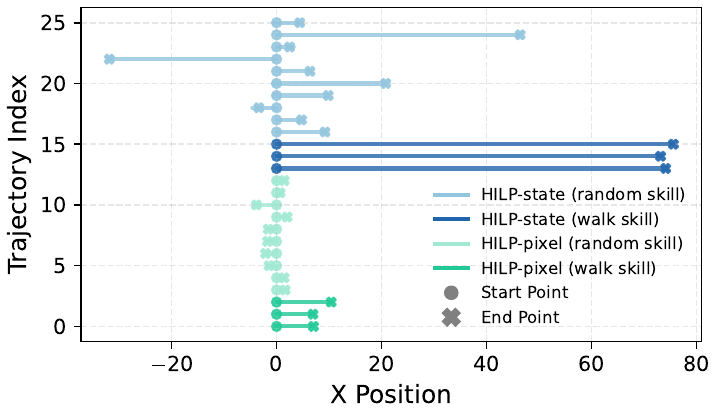} 
\caption{Trajectory comparison of methods with random and walking skills in the walker domain.}
\label{figr4}
\vspace{-.3cm}
\end{figure}

\begin{figure*}[t]
\centering
\includegraphics[width=2.1\columnwidth]{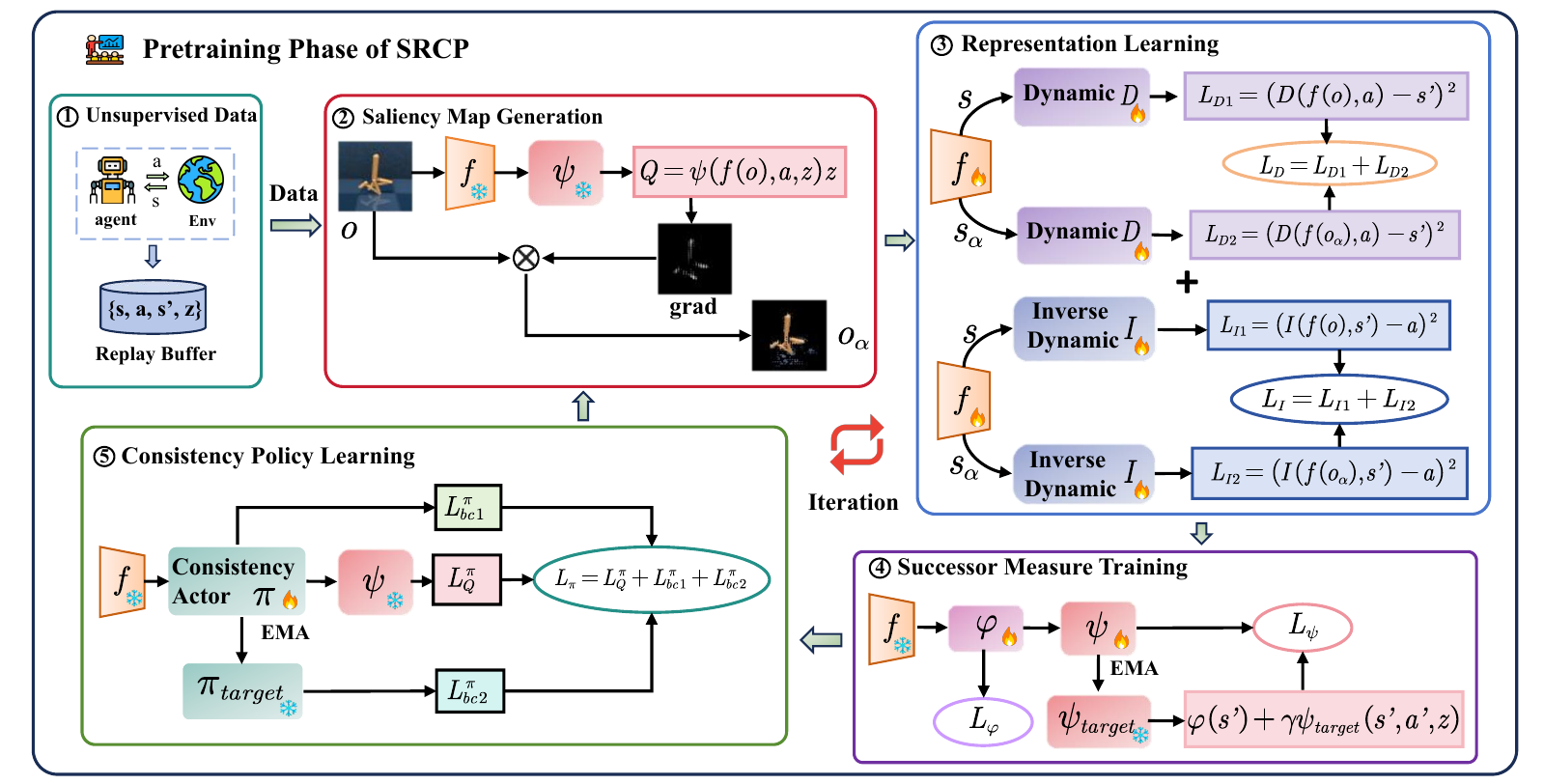} 
\caption{SRCP pretraining framework.
SRCP first leverages unsupervised data to generate saliency maps that guide the learning of saliency-aware dynamic representations. The resulting encoder is shared between successor measure training and consistent policy learning, enabling effective successor measure and expressive policy behaviors, thereby enhancing generalization.}
\label{figr5}
\vspace{-.35cm}
\end{figure*}

In visual URL, SR must learn skill-conditioned policies from latent representations derived from high-dimensional visual inputs. 
Such representations may be suboptimal, which can hinder the learning of expressive and controllable skill-conditioned policies. 
To evaluate this effect, we compare HILP-state (trained in compact state space) and HILP-pixel (trained on raw visual inputs) in the \textit{Walker} domain of DMC~\cite{dmc}.
Fig.\ref{figr4} visualizes trajectories under two conditions: (1) random skill vectors and (2) a skill inferred for the walking task. 
Under random skills, HILP-state produces diverse and meaningful trajectories, demonstrating effective skill expressiveness. 
HILP-pixel, in contrast, exhibits limited and less diverse movements. 
When conditioned on the inferred walking skill, HILP-state achieves consistent long-range locomotion, whereas HILP-pixel produces short and unstable walking behaviors, indicating difficulty in skill controllability.
These observations reveal a fundamental gap between state-based and visual URL settings. 
While structured states enable direct and effective skill-conditioned policy learning, the reliance on learned visual representations in visual URL \textbf{poses challenges in modeling multi-modal skill-conditioned action distributions and ensuring skill controllability, thereby limiting generalization}.

\vspace{-.1cm}
\section{Methodology}
\vspace{-.1cm}
\subsection{Framework Overview}
\vspace{-.1cm}
To improve generalization, we propose SRCP, a visual SR framework that learns effective representations and multi-modal skill-conditioned policies.
As illustrated in Fig.\ref{figr5}, SRCP consists of five components optimized in an iterative training loop:
(1) \textbf{Unsupervised Datasets}: Diverse task-agnostic trajectories are used as the unsupervised training source.
(2) \textbf{Saliency Map Generation}: At each iteration, saliency maps are computed to highlight attention regions of the encoder and successor measure.
(3) \textbf{Representation Learning}: The encoder is updated via a saliency-guided dynamics task, encouraging it to focus on dynamics-relevant features guided by saliency maps.
(4) \textbf{Successor Measure Training}: The basic feature and successor feature are jointly optimized with the updated encoder.
(5)\textbf{ Consistency Policy Learning}: A skill-conditioned consistency policy is trained with classifier-free guidance to model multi-modal skills and ensure skill controllability.

\subsection{Saliency Map Generation}
To encourage effective representations that focus on dynamics-relevant regions, we generate saliency maps from the gradients of the value function derived from successor features.
Given an input observation \( o \), the encoder \( f \) extracts a representation \( s=f(o) \), which is used to compute both the basic feature \( \varphi(f(o)) \) and the successor feature \( \psi(f(o), a, z) \). The value function is defined as \( Q = \psi(f(o), a, z)^\top z \).
We then compute the gradient of \( Q \) with respect to the input observation \( o \), and construct a saliency map \( o_\alpha \) by retaining only the top-ranked pixels based on gradient magnitude, while masking out less informative regions. 

\subsection{Saliency Dynamic Representation Learning}
To facilitate effective representation learning for successor measures, we decouple it from the SR objective by introducing a saliency-guided dynamic representation module based on saliency maps. 
To encourage the encoder to capture dynamics-relevant representations, we incorporate both a forward and an inverse dynamics model to train the encoder.
The forward dynamics model \( D \) predicts the next-state representation \( s' \) from the current observation and action. The forward dynamics loss is defined as:
\begin{equation}
    \mathcal{L}_{D1} = \left\| D(f(o), a) - s' \right\|^2,
\end{equation}
where \( f(o) \) is the representation of the current observation, \( a \) is the action, and \( s'\) is the representation of the next observation.
The inverse dynamics model \( I \) predicts the action \( a \) given the current and next representations. The inverse dynamics loss is:
\begin{equation}
    \mathcal{L}_{I1} = \left\| I(f(o), s') - a \right\|^2.
\end{equation}
To further encourage the encoder to focus on dynamics-relevant regions within observations, we leverage the saliency-masked observation \( o_\alpha \) when training both the forward and inverse dynamics models.
The forward dynamics model predicts the representation of the next observation \( s' \) using the saliency representation \( f(o_\alpha) \) and the action \( a \). The saliency forward dynamics loss is defined as:
\begin{equation}
    \mathcal{L}_{D2} = \left\| D(f(o_\alpha), a) - s' \right\|^2.
\end{equation}
Similarly, the inverse dynamics model takes the representations of the saliency map and the next observation as input to predict the action. The saliency inverse dynamics loss is:
\begin{equation}
    \mathcal{L}_{I2} = \left\| I(f(o_\alpha), s') - a \right\|^2.
\end{equation}
By minimizing the saliency losses, the encoder is encouraged to predict the underlying dynamics from observations and their corresponding saliency maps, thereby guiding the agent to focus on dynamics-relevant regions and improving representation. The total representation loss is defined as:
\begin{equation}
    \mathcal{L}_{\text{rep}} = \mathcal{L}_{D1}  + \mathcal{L}_{I1} + \beta*(\mathcal{L}_{D2} + \mathcal{L}_{I2}).
\end{equation}
This saliency-guided representation learning encourages SRCP to capture dynamics-relevant features that facilitate successor training and enhance generalization in visual SR.


\subsection{Successor Measure Training}
While various features can be employed, we employ a Hilbert space representation~\cite{hilp} to learn the basic feature $\varphi$. The learned basic feature is further used to update the successor feature $\psi$. 
Specifically, the successor feature is trained to satisfy Bellman consistency, formulated as:
\begin{equation}
L_\psi = \left\| \psi(s, a, z) -  \varphi(s') - \gamma\bar{\psi}(s', a', z) \right\|^2,
\end{equation}
where $\bar{\psi}(\cdot)$ denotes the target successor feature learner.



\subsection{Consistency Policy Learning}
Visual URL poses challenges in modeling multi-modal skill-conditioned action distributions while ensuring skill controllability.
While consistency models demonstrate strong flexibility in capturing multi-modal behaviors, effectively integrating skill conditions to realize controllable skill policies in visual URL remains an open challenge.
To address this, we propose a skill-conditioned consistency policy with URL-specific classifier-free guidance to enhance both behavioral diversity and controllability.
We formulate the policy as a denoising process, where a network 
\( g_\theta \) learns to recover the clean action \( a_0 \) from a noisy version \( a_t \), conditioned on both the state $s$ and a skill vector $z$:
\begin{equation}
 \quad g_\theta(s, a_t, z) \approx a_0, \forall t.
\end{equation}
This formulation enables consistent behavior across noise levels and allows the policy to model the conditional distribution \( p(a_0 \mid s, z) \).
To balance diversity and controllability, we adopt a classifier-free guidance strategy inspired by diffusion models. Rather than using the conditioned output $g_\theta(s, a_t, z)$ directly, the final action is computed as:
\begin{equation}
a = g_\theta(s, a_t, \varnothing) + \omega \cdot \left( g_\theta(s, a_t, z) - g_\theta(s, a_t, \varnothing) \right),
\end{equation}
where $a$ is the action, $\varnothing$ denotes the unconditional skill input, and $\omega$ is the guidance weight controlling the strength of conditioning.
Notably, the state $s$ is included in both conditioned and unconditioned branches, enabling the model to capture state-dependent multi-modal action distributions while ensuring that actions are primarily driven by skills. 

\begin{table*}[h]
\vspace{-0.45cm}
\setlength{\tabcolsep}{10pt}
\renewcommand\arraystretch{1.1}
\centering
\resizebox{\textwidth}{!}{
\begin{tabular}{ccccccccccc}
\toprule
Domain & Task & AE~\cite{autoencoder} & CL~\cite{cl} & LAP~\cite{lap}  & LRA-SR~\cite{lap}  & FDM~\cite{fb1} & FB~\cite{fb1} & HILP~\cite{hilp} & SRCP \\
\midrule
\multirow{5}{*}{Walker} 
& Stand &  $614 \pm 85$ &  $255 \pm 95$ & $343 \pm 128$ & $246 \pm 114$  & $697 \pm 120$ & $234 \pm 62$ & $424 \pm 58$ & $\textbf{752} \pm \textbf{72}$ \\
& Walk  &  $301 \pm 122$ & $72 \pm 39$ & $87 \pm 41$ & $88 \pm 67$ & $459 \pm 91$ & $100 \pm 48$  & $260 \pm 151$ & $\textbf{486} \pm \textbf{42}$ \\
& Run   &  $127 \pm 24$ & $46 \pm 17$  & $60 \pm 21$ & $52 \pm 25$ & $195 \pm 60$ & $61 \pm 25$ & $112 \pm 51$ & $\textbf{204} \pm \textbf{21}$ \\
& Flip  &  $224 \pm 95$ & $83 \pm 37$ & $71 \pm 24$ & $62 \pm 32$ & $251 \pm 73$ & $63 \pm 16$ & $154 \pm 62$ & $\textbf{369} \pm \textbf{46}$ \\
\cline{2-10}
& Average  & $317$ & $114$ & $140$ & $112$ & {$401$} & $115$ & $238$ & $\textbf{453}$ \\
\midrule
\multirow{5}{*}{Quadruped}
& Stand & $351 \pm 136$ & $105 \pm 19$ & $265 \pm 123$ & $390 \pm 141$ & $355 \pm 80$ & $270 \pm 91$ & $344 \pm 150$ & $\textbf{514} \pm \textbf{111}$ \\
& Walk  & $180 \pm 66$ & $51 \pm 8$ & $143 \pm 70$  & $208 \pm 71$ & $164 \pm 34$ & $141 \pm 57$ & $174 \pm 78$ & $\textbf{268} \pm \textbf{42}$ \\
& Run   & $167 \pm 69$ & $53 \pm 6$ & $136 \pm 79$ & $192 \pm 66$ & $168 \pm 29$ & $131 \pm 44$ & $169 \pm 77$ & $\textbf{265} \pm \textbf{59}$ \\
& Jump  & $238 \pm 70$ & $77 \pm 14$ & $173 \pm 71$ & $272 \pm 97$ & $237 \pm 38$ & $190 \pm 62$ & $239 \pm 98$ & $\textbf{374} \pm \textbf{108}$ \\
\cline{2-10}
& Average & $234$ & $72$ & $179$ & {$266$} & $231$ & $183$ & $232$ & $\textbf{355}$ \\
\midrule
\multirow{5}{*}{Cheetah}
& Walk & $503 \pm 179$ & $115 \pm 94$ & $235 \pm 197$ & $294 \pm 134$ & $493 \pm 219$ & $206 \pm 222$ & $690 \pm 255$ & $\textbf{805} \pm \textbf{110}$ \\
& Walk Backward & $139 \pm 101$ & $135 \pm 109$ & $102 \pm 126$ & $189 \pm 103$ & $444 \pm 82$ & $344 \pm 333$ & $712 \pm 152$ & $\textbf{853} \pm \textbf{70}$ \\
& Run & $179 \pm 91$ & $30 \pm 16$ & $57 \pm 67$ & $67 \pm 29$ & $176 \pm 59$ & $79 \pm 82$ & $218 \pm 83$ & $\textbf{269} \pm \textbf{58}$ \\
& Run Backward & $50 \pm 39$ & $27 \pm 21$ & $19 \pm 24$ & $35 \pm 25$ & $99 \pm 51$ & $108 \pm 103$ & $197 \pm 49$ & $\textbf{243} \pm \textbf{43}$ \\
\cline{2-10}
& Average & $218$ & $77$ & $103$ & $146$ & $303$ & $184$ & $454$ & $\textbf{543}$ \\
\midrule
\multirow{5}{*}{Jaco}
& Reach Top Left & $31 \pm 3$ & $4 \pm 1$ & $25 \pm 12$ & $22 \pm 16$ & $37 \pm 9$ & $45 \pm 40$ & $27 \pm 7$ & $\textbf{48} \pm \textbf{12}$ \\
& Reach Top Right & $26 \pm 6$ & $4 \pm 3$ & $36 \pm 10$ & $19 \pm 13$ & $46 \pm 11$ & $40 \pm 19$ & $35 \pm 12$ & $\textbf{47} \pm \textbf{10}$ \\
& Reach Bottom Left & $24 \pm 8$ & $5 \pm 3$ & $29 \pm 10$ & $14 \pm 8$ & $18 \pm 4$ & $\textbf{39} \pm \textbf{11}$ & $29 \pm 6$ & $31 \pm 8$ \\
& Reach Bottom Right & $17 \pm 4$ & $4 \pm 5$ & $35 \pm 6$ & $16 \pm 12$ & $26 \pm 6$ & $\textbf{37} \pm \textbf{12}$ & $\textbf{37} \pm \textbf{13}$ & $\textbf{37} \pm \textbf{7}$ \\
\cline{2-10}
& Average & $25$ & $4$ & $31$ & $18$ & $32$ & $40$ & {$32$} & $\textbf{41}$ \\
\bottomrule
\end{tabular}}
\caption{Zero-shot generalization results across 16 tasks.  
Each score is averaged over 4 datasets and 4 seeds (i.e. 16 runs).}
\vspace{-0.35cm}
\label{table_zeroshot_new}
\end{table*}

The consistency policy is trained with three key objectives.
First, we incorporate a skill-conditioned value objective that encourages the policy to maximize expected skill returns and improve skill controllability under the successor feature formulation:
\begin{equation}
\mathcal{L}^{\pi}_{Q} = \mathbb{E}_{(s,z)\sim \mathcal{D}}[-\psi(s, \pi(s, z), z)^\top z].
\end{equation}
Second, to stabilize training and mitigate distributional shift in offline RL, we incorporate a skill-conditioned behavior consistency loss. This loss regularizes the policy by encouraging consistent outputs on Gaussian-noise perturbed dataset actions across different noise levels:
\begin{equation}
\mathcal{L}^{\pi}_{bc1} = \mathbb{E}_{(s,z)\sim \mathcal{D}} \left\| g_\theta(s, a_{dt_1}, z) - g_\theta(s, a_{dt_2}, z) \right\|^2,
\end{equation}
where \( a_{dt} = a_{\mathrm{data}} + \sigma_t \epsilon \), \( \epsilon \sim \mathcal{N}(0, I) \), \( a_{\mathrm{data}} \) denotes the dataset action and \( \sigma_t \) is the noise scale at level \( t \).
Third, an unconditional behavior consistency loss is introduced to promote expressive multi-modal action distribution modeling.
Specifically, the unconditional branch $g_\theta(s, a_t, \varnothing)$ is trained to produce consistent outputs on noisy actions sampled from a random skill policy, encouraging diverse multi-modal skill-conditioned behaviors:
\begin{equation}
\mathcal{L}^{\pi}_{bc2} = \mathbb{E}_{s\sim \mathcal{D}} \left\| g_\theta(s, a_{rt_1}, \varnothing) - g_\theta(s, a_{rt_2}, \varnothing) \right\|^2,
\end{equation}
where \( a_{rt} = \pi_{\mathrm{target}}(s, z') + \sigma_t \epsilon \) denotes a noise-perturbed action sampled from the random skill policy \(\pi_{\mathrm{target}}(s, z')\).
The final training objective combines the three components:
\begin{equation}
\mathcal{L}_{\pi} = \mathcal{L}^{\pi}_{Q} + \lambda_1 \mathcal{L}^{\pi}_{bc1} + \lambda_2 \mathcal{L}^{\pi}_{bc2},
\end{equation}
where $\lambda_1$ and $\lambda_2$ are weighting coefficients. 
Jointly optimizing these objectives encourages the policy to learn skill controllable and multi-modal behaviors, which are critical for generalization in visual URL.
\vspace{-.1cm}
\section{Experiments}
\vspace{-.1cm}
In this section, we conduct comprehensive experiments to evaluate the effectiveness of the proposed SRCP framework in visual URL tasks. Our experiments aim to answer the following key research questions: (1) Does SRCP improve task generalization in visual URL? (2) How does each component of SRCP contribute to task generalization? 
(3) Can SRCP be integrated into other SR methods?
(4) How do hyperparameters affect generalization?


\subsection{Benchmark}
We adopt the URL Benchmark~\cite{urlb} and ExORL datasets~\cite{exorl} to evaluate the task generalization performance in visual URL. 
The benchmark consists of 16 visual continuous control tasks spanning 4 domains, each pretrained with 4 different datasets: RND, PROTO, APS, and APT.
We compare SRCP against several state-of-the-art (SOTA) zero-shot task generalization methods. 
Specifically, we evaluate SRCP in comparison with:
(1) Successor feature (SF) methods utilizing different basic feature learning techniques, including AE~\cite{autoencoder}, CL~\cite{cl}, Lap~\cite{lap}, LRA-SR~\cite{lap}, FDM~\cite{fb1}, and HILP~\cite{hilp} representations; and
(2) the Forward-Backward (FB)~\cite{fb1} representation, a SOTA method for zero-shot URL. See Appendices D, E, F for settings, and baselines.


\subsection{Does SRCP improve generalization?}
As shown in Table~\ref{table_zeroshot_new}, we evaluate the generalization performance of SRCP and baseline methods across 16 tasks spanning 4 domains.
Each result is averaged over 16 runs, obtained from 4 datasets with 4 random seeds each. 
In SRCP, we adopt Hilbert Representations for basic feature learning.
SRCP consistently achieves superior or near-optimal performance across all tasks.
Specifically, SRCP outperforms the current SOTA by 13\%, 33\%, and 11\% in the Walker, Quadruped, and Cheetah domains, respectively.
Moreover, the method achieves the best average results across all four datasets, demonstrating strong robustness to dataset variations.
The comparison with HILP further underscores the architectural improvements SRCP brings to SR methods.
Full experimental results are provided in Appendix A.



\begin{table}[t]
\setlength{\tabcolsep}{6pt}
\renewcommand\arraystretch{1.2}
\centering
\resizebox{0.45\textwidth}{!}{
\begin{tabular}{ccccc}
\toprule
Domain & HILP & SRCP w/o SE & SRCP w/o CP & SRCP \\
\midrule
Walker    & $231$ & $345$ & $396$ & $\mathbf{439}$ \\
Quadruped & $305$ & $352$ & $406$ & $\mathbf{485}$ \\
Cheetah   & $599$ & $600$ & $598$ & $\mathbf{602}$ \\
Jaco      & $34$  & $44$  & $43$  & $\mathbf{50}$ \\
\bottomrule
\end{tabular}}
\caption{Ablation study of SRCP on the RND dataset across 4 domains. 
Each score is the averaged return over 4 tasks per domain, with 4 random seeds per task (i.e., 16 runs).}
\label{table_zeroshot_ablation}
\vspace{-0.4cm}
\end{table}

\vspace{-.05cm}
\subsection{How components contribute to generalization?}
To explore the contribution of each component in SRCP to task generalization, we conduct an ablation study on 16 tasks from the RND dataset using 4 random seeds.
As shown in Table~\ref{table_zeroshot_ablation}, HILP jointly training the encoder and successor components using the Hilbert basic feature objective and successor feature objective.
SRCP w/o SE removes saliency-guided dynamics representation learning from SRCP, while SRCP w/o CP excludes the consistency policy.
The results demonstrate that SRCP w/o SE outperforms HILP, suggesting that consistency policy learning enhances both action diversity and skill controllability, further benefiting generalization. 
SRCP w/o CP also surpasses HILP, indicating the effectiveness of saliency-guided dynamics representation learning in improving generalization. 
When combined, SRCP achieves the superior performance, demonstrating that improving policy modeling on effective representations significantly enhances generalization, highlighting both components are critical in visual URL tasks.

\begin{figure}[t]
\centering
\includegraphics[width=.95\columnwidth]{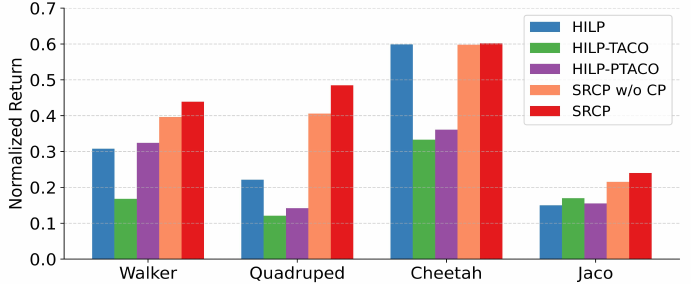} 
\caption{Generalization performance of HILP with various representation learning methods.}
\label{figr7}
\vspace{-.3cm}
\end{figure}

\begin{figure}[t]
\centering
\includegraphics[width=1.0\columnwidth]{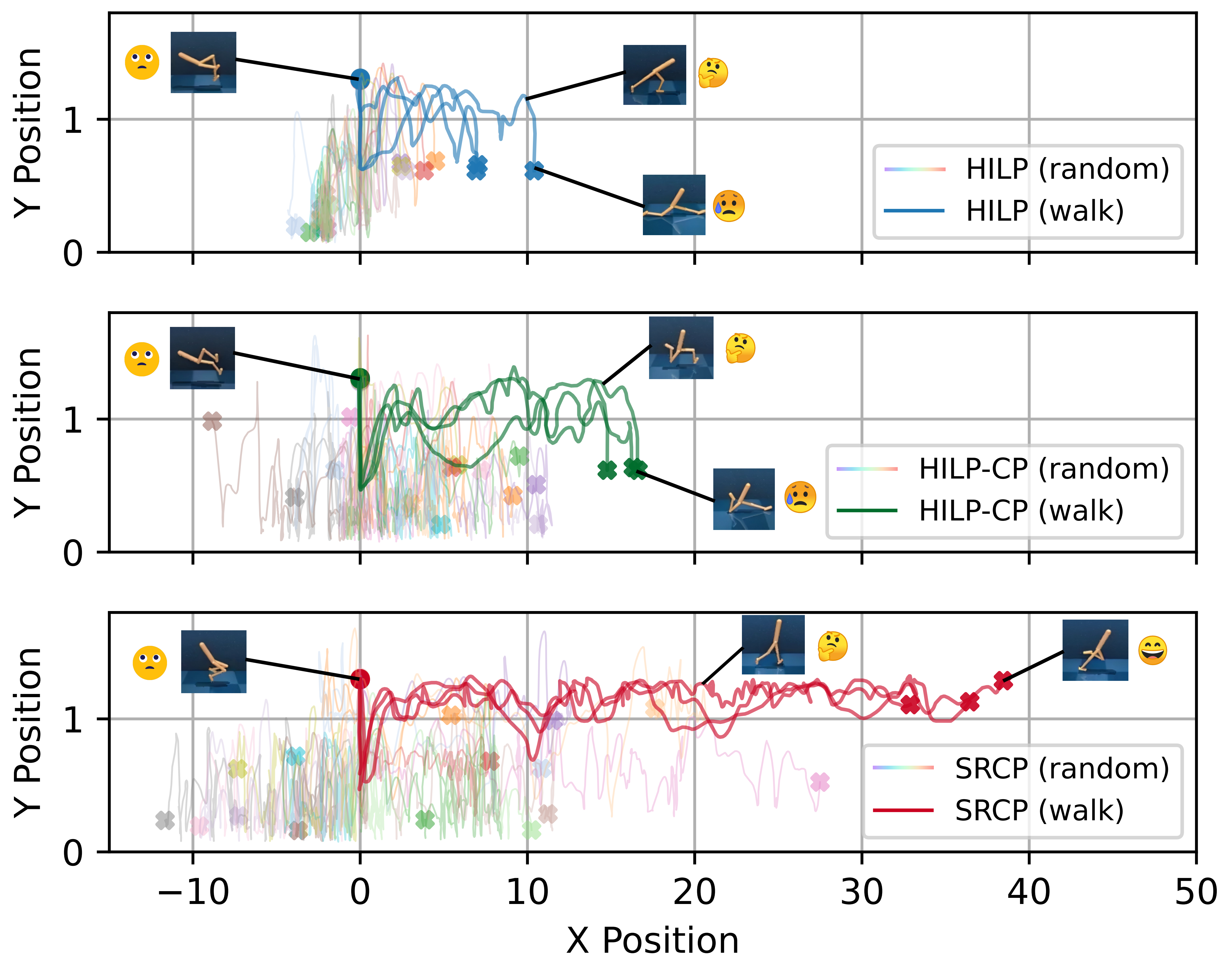} 
\caption{2D trajectories of methods in walker domain.
Solid lines represent trajectories with walking skill, while dashed colored lines are trajectories with random skills.}
\label{figr8}
\vspace{-.5cm}
\end{figure}
To evaluate the effectiveness of the saliency-guided dynamics encoder, we conduct experiments based on the HILP. 
We compare HILP-TACO and HILP-PTACO, which integrate the SOTA contrastive representation learning methods TACO\cite{taco} and Premier-TACO\cite{tacop}, respectively, into HILP.
TACO and Premier-TACO learn representations through temporal contrastive learning, with Premier-TACO proposing a negative example sampling strategy to enhance representation.
We further evaluate two SRCP variants that replace the HILP encoder with our saliency-guided dynamics representation, with and without the consistency policy (denoted as SRCP and SRCP w/o CP).
As Fig.\ref{figr7}, results show that while HILP-TACO and HILP-PTACO improve over HILP in some domains, they underperform in others, indicating that decoupling representation learning from SR objectives with standard representation methods has limited effectiveness.
In contrast, both SRCP w/o CP and SRCP outperform baselines consistently, demonstrating the superior effectiveness of saliency representation for visual URL.

To examine the effect of the consistency policy, we visualize trajectories in the Walker domain with random and walking skills (Fig.\ref{figr8}).
We compare three variants: HILP, HILP-CP (HILP with consistency policy), and SRCP.
HILP shows limited trajectory diversity under random skills and only short-distance walking under the walking skill, reflecting poor skill expressiveness and controllability.
HILP-CP improves trajectory diversity and enables longer walking distances, demonstrating more expressive and controllable behaviors. 
However, its performance remains unstable due to suboptimal representations.
SRCP achieves both multi-modal skills and consistent long-distance locomotion under walking skill, indicating that combining consistency policy with dynamics-relevant representations is essential for modeling multi-modal skill-conditioned action distributions and achieving skill controllability.


\vspace{-.2cm}
\subsection{Can SRCP integrated into other SR methods?}
\begin{figure}[t]
\centering
\includegraphics[width=1.0\columnwidth]{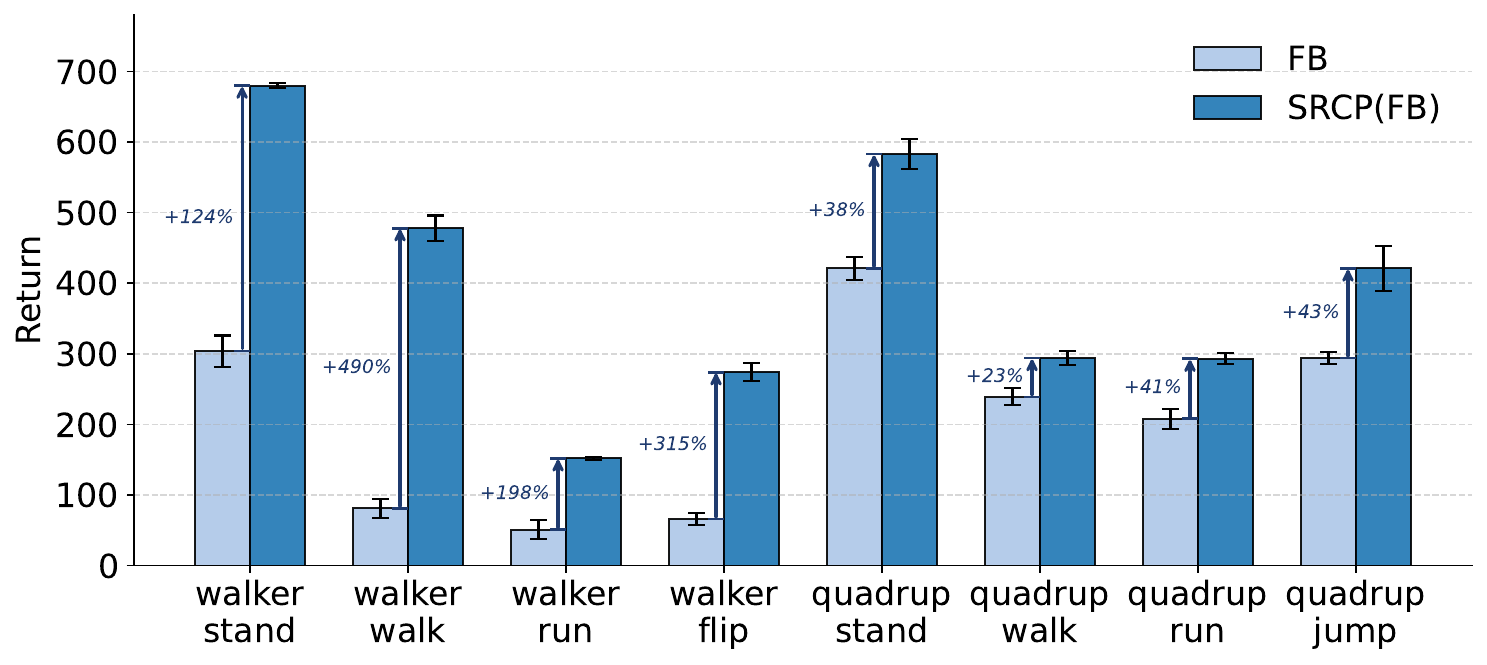} 
\caption{Generalization performance of FB and SRCP(FB).}
\label{figr9}
\vspace{-.4cm}
\end{figure}


To further evaluate the generality of SRCP across different successor representation methods, we compare the performance of FB and SRCP(FB), where FB is adopted as the successor measure learner within the SRCP framework. 
As shown in Fig.\ref{figr9}, we conduct experiments on eight tasks from the walker and quadruped domains in the RND dataset. 
SRCP(FB) consistently outperforms FB in all tasks, benefiting from dynamics-relevant representations, multi-modal action modeling, and enhanced skill controllability introduced by SRCP.
These results demonstrate that SRCP serves as a general and effective framework for improving the generalization capability of SR methods in visual URL.

\vspace{-.05cm}
\subsection{How hyperparameters affect generalization?}
We analyze the sensitivity of SRCP to two key hyperparameters, $\omega$ and $\beta$.
As shown in Table~\ref{table_zeroshot_ablation1}, $\omega$ controls the strength of skill-conditioned guidance. Removing this term ($\omega=0$) leads to a significant drop in performance, while moderate values improve generalization. However, overly large $\omega$ over-constrains policy behavior and slightly degrades performance.
Table~\ref{table_zeroshot_ablation2} examines the effect of $\beta$, which weights the importance of saliency guidance in representation learning. Increasing $\beta$ enhances generalization up to a point, after which excessive focus on salient regions reduces the ability to capture dynamics cues.
Overall, SRCP exhibits stable performance across a wide range of hyperparameters, with balanced settings of $\omega$ and $\beta$ achieving the best trade-off between skill guidance and saliency guidance.

\begin{table}[t]
\vspace{-0.2cm}
\setlength{\tabcolsep}{6pt}
\renewcommand\arraystretch{1.2}
\centering
\resizebox{0.47\textwidth}{!}{
\begin{tabular}{cccccc}
\toprule
Quadruped & $\omega=0$ & $\omega=2$ & $\omega=3$ & $\omega=4$& $\omega=6$ \\
\midrule
Stand    & $370 \pm 38$  & $657 \pm 45$ & ${703} \pm {22}$ & $\textbf{724} \pm \textbf{21}$ & $698 \pm 28$ \\
Walk     & $184 \pm 42$  & $267 \pm 32$ & $\textbf{339} \pm \textbf{27}$ & $337 \pm 27$ & $321 \pm 14$ \\
Run      & $164 \pm 20$  & $289 \pm 13$ & $\textbf{361} \pm \textbf{24}$ & $334 \pm 14$ & $360 \pm 18$ \\
Flip     & $218 \pm 45$  & ${452} \pm {45}$ & $\textbf{535} \pm \textbf{14}$ & $451 \pm 18$ & $442 \pm 55$ \\
\cline{1-6}
Average     & $234$  & $416$ & $\textbf{485}$ & $462$ & $455$ \\
\bottomrule
\end{tabular}}
\caption{Ablation study of parameter $\omega$ in SRCP on Quadruped domain in RND dataset, with 4 random seeds per task.}
\label{table_zeroshot_ablation1}
\vspace{-0.2cm}
\end{table}

\begin{table}[t]
\setlength{\tabcolsep}{6pt}
\renewcommand\arraystretch{1.2}
\centering
\resizebox{0.45\textwidth}{!}{
\begin{tabular}{ccccc}
\toprule
Walker & $\beta=0$ & $\beta=0.2$ & $\beta=0.5$ & $\beta=1$ \\
\midrule
Stand    & $544 \pm 20$  & $550 \pm 26$ & $\textbf{671} \pm \textbf{51}$ & $583 \pm 18$  \\
Walk     & $354 \pm 42$  & $458 \pm 19$ & $\textbf{520} \pm \textbf{32}$ & $484 \pm 48$ \\
Run      & $156 \pm 12$  & $170 \pm 35$ & $\textbf{194} \pm \textbf{38}$ & $169 \pm 23$ \\
Flip     & $310 \pm 15$  & $315 \pm 11$ & $\textbf{369} \pm \textbf{25}$ & $346 \pm 5$ \\
\cline{1-5}
Average     & $341$  & $373$ & $\textbf{439}$ & $396$ \\
\bottomrule
\end{tabular}}
\caption{Ablation study of parameter $\beta$ in SRCP on Walker domain in RND dataset, with 4 random seeds per task.}
\label{table_zeroshot_ablation2}
\vspace{-0.5cm}
\end{table}
\vspace{-.1cm}
\section{Conclusion}
\vspace{-.1cm}
This paper presents a general framework to improve the zero-shot task generalization of SR methods in visual URL. 
We first demonstrate that SR methods often extract suboptimal representations, which impair successor measure and skill-conditioned policy modeling.
To address this issue, we propose the SRCP framework, which capture dynamics-relevant features with saliency-guided dynamics representation learning.
In addition, SRCP introduces a consistency policy tailored for URL tasks, which leverages classifier-free guidance to model multi-modal skills and improve skill controllability.
SRCP is evaluated on 16 visual tasks across 4 datasets and demonstrated superior zero-shot generalization.
Moreover, it is compatible with various SR methods to improve generalization in visual URL.
While SRCP shows strong generalization in offline visual URL, extending SR methods to online settings remains a promising direction. 

{
    \small
    \bibliographystyle{ieeenat_fullname}
    \bibliography{main}
}
\onecolumn
\appendix
\section*{\centering Appendix}

\addcontentsline{toc}{section}{Appendix} 

\section{Full results on the ExORL benchmark}
We evaluate the generalization performance of SRCP and baseline methods across 16 visual continuous control tasks from 4 domains in the ExORL benchmark, using 4 different datasets for pretraining.
All methods are pretrained from raw pixel observations for 500k steps. During evaluation, 10k transitions of task-specific visual-based data are provided to infer the skill vectors for downstream tasks, enabling zero-shot generalization.
Fig.\ref{figa1} (a) summarizes the overall performance of SRCP and baseline methods across all 16 tasks from 4 domains with 4 datasets with 4 seeds. The results show that SRCP achieves SOTA overall performance, demonstrating superior zero-shot generalization ability.
Fig.\ref{figa1} (b) provides a detailed breakdown of performance across 4 domains. The results highlight that SRCP achieves the best or near-best performance in each domains. Furthermore, the comparison between SRCP and HILP highlights the architectural improvements that SRCP brings to SR methods.
\begin{figure*}[h]
\centering
\includegraphics[width=1.0\columnwidth]{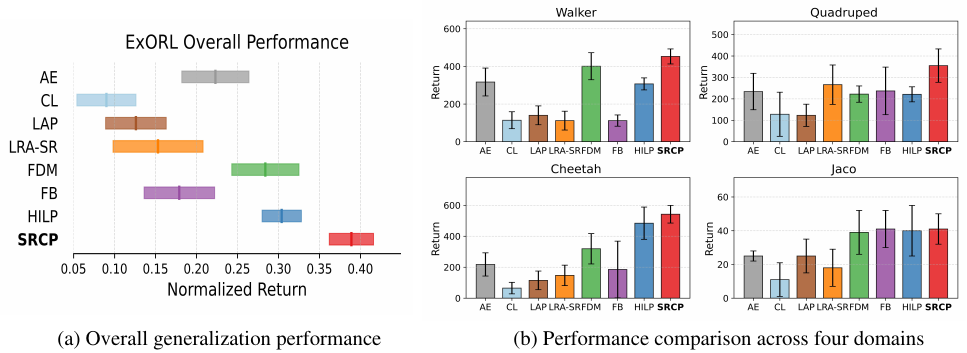} 
\caption{Zero-shot generalization performance on visual tasks.
(a) Overall performance of each method evaluated on 4 datasets, 4 domains, and 4 tasks per domain using 4 random seeds (i.e., 256 values in total).
(b) Performance in 4 domains.}
\label{figa1}
\vspace{-.3cm}
\end{figure*}

\clearpage
Fig.\ref{figr2a} illustrates the comparative performance of the SRCP method against baseline approaches across all tasks. 
Notably, SRCP consistently attains either superior or near-optimal performance across all tasks, highlighting its effectiveness across a diverse range of tasks.
The complete results (unnormalized returns) are reported in Table~\ref{table_zeroshot_new_all}, with all methods evaluated using 4 random seeds.

\begin{figure}[h]
\centering
\includegraphics[width=1.0\columnwidth]{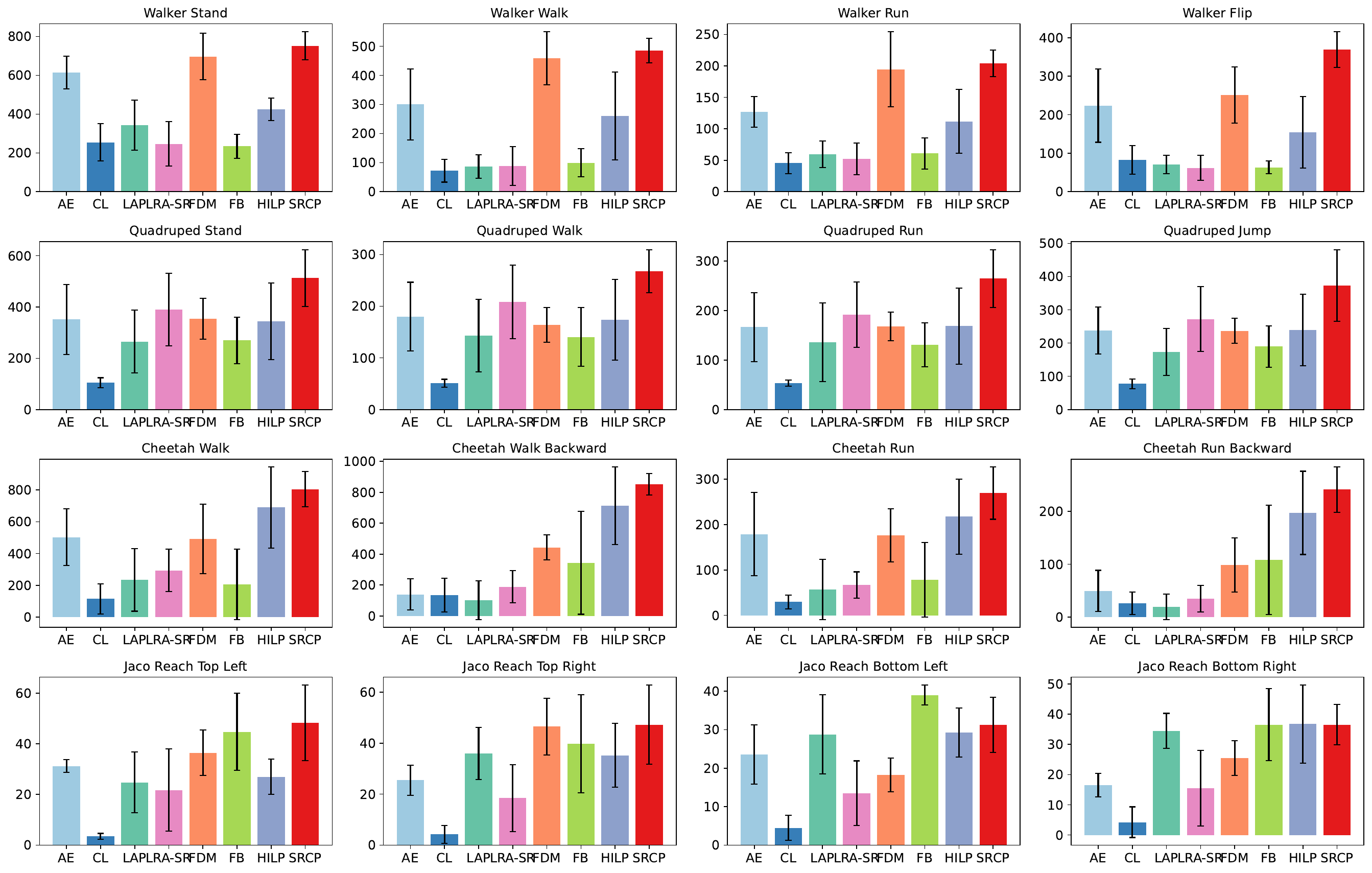} 
\caption{Experimental results on the pixel-based ExORL benchmark for each task, aggregated over four datasets and four random seeds (i.e., 16 runs in total).}
\label{figr2a}
\end{figure}

\begin{table*}[h]
\vspace{-0.4cm}
\caption{Zero-shot Performance across Multiple Domains}
\setlength{\tabcolsep}{5pt}
\renewcommand\arraystretch{1.15}
\centering
\resizebox{0.7\textwidth}{!}{
\begin{tabular}{cccccccccccc}
\toprule
Dataset & Domain & Task & AE & CL & LAP  & LRA-SR  & FDM & FB &HILP &  SRCP \\
\midrule
\multirow{20}{*}{RND}
& \multirow{5}{*}{Walker} 
& Stand &  $546 \pm 30$ &  $144 \pm 4$ & $296 \pm 35$ & $213 \pm 49$  & $557 \pm 99$ & $135 \pm 12$ & $386 \pm 98$ & $671 \pm 51$ \\
&  & Walk  &  $358 \pm 12$ & $36 \pm 1$ & $49 \pm 5$ & $147 \pm 16$ & $452 \pm 52$ & $48 \pm 9$  & $326 \pm 31$ & $520 \pm 32$ \\
&  & Run   &  $125 \pm 7$ & $25 \pm 2$  & $43 \pm 2$ & $82 \pm 3$ & $146 \pm 60$ & $31 \pm 10$ & $121 \pm 20$ & $194 \pm 38$ \\
&  & Flip  &  $265 \pm 14$ & $50 \pm 3$ & $35 \pm 2$ & $101 \pm 14$ & $282 \pm 52$ & $36 \pm 15$ & $92 \pm 35$ & $369 \pm 25$ \\
\cline{3-11}
&  & Average  & $324$ & $64$ & $106$ & $147$ & \underline{$359$} & $63$ & $231$ & $\textbf{439}$ \\
\cline{2-11}
& \multirow{5}{*}{Quadruped}
& Stand & $467 \pm 31$ & $114 \pm 12$ & $446 \pm 18$ & $476 \pm 37$ & $374 \pm 85$ & $257 \pm 42$ & $455 \pm 84$ & $703 \pm 22$ \\
&  & Walk  & $248 \pm 21$ & $58 \pm 4$ & $252 \pm 26$  & $274 \pm 4$ & $199 \pm 63$ & $104 \pm 12$ & $254 \pm 31$ & $339 \pm 27$ \\
&  & Run   & $228 \pm 30$ & $59 \pm 7$ & $252 \pm 35$ & $239 \pm 18$ & $192 \pm 29$ & $105 \pm 36$ & $241 \pm 12$ & $361 \pm 24$ \\
&  & Jump  & $274 \pm 55$ & $88 \pm 12$ & $270 \pm 11$ & $340 \pm 31$ & $273 \pm 66$ & $164 \pm 22$ & $269 \pm 28$ & $535 \pm 14$ \\
\cline{3-11}
&  & Average & $304$ & $80$ & $305$ & \underline{$332$} & $260$ & $158$ & $305$ & $\textbf{485}$ \\
\cline{2-11}
& \multirow{5}{*}{Cheetah}
& Walk & $658 \pm 98$ & $275 \pm 10$ & $181 \pm 15$ & $445 \pm 95$ & $470 \pm 182$ & $559 \pm 102$ & $895 \pm 33$ & $859 \pm 34$ \\
&  & Walk Backward & $5 \pm 2$ & $100 \pm 6$ & $25 \pm 9$ & $196 \pm 8$ & $441 \pm 107$ & $803 \pm 58$ & $927 \pm 35$ & $934 \pm 45$ \\
&  & Run & $258 \pm 16$ & $51 \pm 8$ & $40 \pm 7$ & $106 \pm 33$ & $178 \pm 41$ & $211 \pm 34$ & $276 \pm 46$ & $322 \pm 44$ \\
&  & Run Backward & $3 \pm 3$ & $19 \pm 3$ & $5 \pm 1$ & $24 \pm 2$ & $126 \pm 9$ & $243 \pm 18$ & $297 \pm 46$ & $293 \pm 42$ \\
\cline{3-11}
&  & Average & $231$ & $111$ & $63$ & $193$ & $304$ & $454$ & \underline{$599$} & $\textbf{602}$ \\
\cline{2-11}
& \multirow{5}{*}{Jaco}
& Reach Top Left & $28 \pm 3$ & $5 \pm 3$ & $40 \pm 8$ & $38 \pm 4$ & $40 \pm 16$ & $45 \pm 14$ & $23 \pm 9$ & $53 \pm 9$ \\
&  & Reach Top Right & $20 \pm 7$ & $10 \pm 4$ & $24 \pm 5$ & $20 \pm 4$ & $38 \pm 21$ & $70 \pm 10$ & $39 \pm 5$ & $60 \pm 7$ \\
&  & Reach Bottom Left & $22 \pm 4$ & $10 \pm 2$ & $43 \pm 4$ & $21 \pm 11$ & $20 \pm 14$ & $63 \pm 51$ & $28 \pm 6$ & $43 \pm 8$ \\
&  & Reach Bottom Right & $22 \pm 4$ & $13 \pm 3$ & $35 \pm 5$ & $33 \pm 9$ & $16 \pm 8$ & $51 \pm 9$ & $47 \pm 5$ & $42 \pm 8$ \\
\cline{3-11}
&  & Average & $23$ & $10$ & $36$ & $28$ & $29$ & $\textbf{52}$ & $34$ & \underline{$50$} \\
\midrule
\multirow{20}{*}{Proto}
& \multirow{5}{*}{Walker} 
& Stand & $677 \pm 18$ & $316 \pm 3$ & $389 \pm 7$ & $156 \pm 5$ & $854 \pm 46$ & $241 \pm 32$ & $365 \pm 33$ & $830 \pm 39$ \\
&  & Walk  & $216 \pm 38$ & $125 \pm 2$ & $78 \pm 3$ & $28 \pm 2$ & $563 \pm 268$ & $92 \pm 7$ & $117 \pm 12$ & $504 \pm 66$ \\
&  & Run   & $122 \pm 5$ & $60 \pm 1$ & $67 \pm 1$ & $26 \pm 1$ & $267 \pm 44$ & $48 \pm 6$ & $61 \pm 8$ & $227 \pm 16$ \\
&  & Flip  & $155 \pm 27$ & $132 \pm 4$ & $75 \pm 2$ & $29 \pm 2$ & $212 \pm 65$ & $74 \pm 6$ & $80 \pm 15$ & $428 \pm 75$ \\
\cline{3-11}
&  & Average & $293$ & $158$ & $152$ & $60$ & \underline{$474$} & $114$ & $156$ & $\textbf{497}$ \\
\cline{2-11}
& \multirow{5}{*}{Quadruped}
& Stand & $120 \pm 7$ & $112 \pm 9$ & $138 \pm 10$ & $163 \pm 18$ & $226 \pm 65$ & $193 \pm 14$ & $86 \pm 8$ & $455 \pm 21$ \\
&  & Walk  & $70 \pm 6$ & $56 \pm 3$ & $68 \pm 6$ & $91 \pm 9$ & $108 \pm 39$ & $110 \pm 11$ & $45 \pm 9$ & $234 \pm 3$ \\
&  & Run   & $48 \pm 5$ & $54 \pm 6$ & $52 \pm 4$ & $80 \pm 12$ & $120 \pm 21$ & $108 \pm 12$ & $39 \pm 5$ & $230 \pm 18$ \\
&  & Jump  & $116 \pm 8$ & $71 \pm 3$ & $89 \pm 10$ & $104 \pm 11$ & $182 \pm 22$ & $129 \pm 16$ & $67 \pm 11$ & $238 \pm 68$ \\
\cline{3-11}
&  & Average & $89$ & $73$ & $87$ & $110$ & \underline{$159$} & $135$ & $59$ & $\textbf{289}$ \\
\cline{2-11}
& \multirow{5}{*}{Cheetah}
& Walk & $206 \pm 49$ & $46 \pm 4$ & $74 \pm 41$ & $109 \pm 28$ & $150 \pm 126$ & $20 \pm 21$ & $351 \pm 67$ & $624 \pm 44$ \\
&  & Walk Backward & $117 \pm 24$ & $56 \pm 3$ & $25 \pm 9$ & $59 \pm 11$ & $384 \pm 232$ & $8 \pm 5$ & $692 \pm 26$ & $856 \pm 44$ \\
&  & Run   & $30 \pm 6$ & $7 \pm 1$ & $5 \pm 1$ & $24 \pm 3$ & $87 \pm 67$ & $6 \pm 5$ & $157 \pm 22$ & $174 \pm 43$ \\
&  & Run Backward & $21 \pm 2$ & $10 \pm 1$ & $5 \pm 2$ & $10 \pm 2$ & $41 \pm 22$ & $7 \pm 3$ & $100 \pm 15$ & $208 \pm 12$ \\
\cline{3-11}
&  & Average & $94$ & $30$ & $27$ & $51$ & $166$ & $10$ & \underline{$325$} & $\textbf{466}$ \\
\cline{2-11}
& \multirow{5}{*}{Jaco}
& Reach Top Left & $30 \pm 9$ & $4 \pm 2$ & $32 \pm 8$ & $5 \pm 2$ & $43 \pm 25$ & $47 \pm 12$ & $39 \pm 5$ & $43 \pm 11$ \\
&  & Reach Top Right & $34 \pm 6$ & $2 \pm 1$ & $50 \pm 10$ & $4 \pm 1$ & $48 \pm 21$ & $43 \pm 9$ & $44 \pm 6$ & $57 \pm 5$ \\
&  & Reach Bottom Left & $33 \pm 8$ & $3 \pm 3$ & $34 \pm 6$ & $2 \pm 1$ & $24 \pm 22$ & $38 \pm 11$ & $24 \pm 5$ & $30 \pm 10$ \\
&  & Reach Bottom Right & $18 \pm 7$ & $1 \pm 1$ & $40 \pm 4$ & $1 \pm 1$ & $26 \pm 21$ & $43 \pm 4$ & $29 \pm 4$ & $44 \pm 19$ \\
\cline{3-11}
&  & Average & $29$ & $3$ & $39$ & $3$ & $35$ & $\underline{43}$ & $34$ & $\textbf{44}$ \\
\midrule
\multirow{20}{*}{APS}
& \multirow{5}{*}{Walker} 
& Stand & $516 \pm 12$ & $377 \pm 17$ & $519 \pm 38$ & $441 \pm 13$ & $608 \pm 90$ & $304 \pm 22$ & $429 \pm 26$ & $689 \pm 24$ \\
&  & Walk  & $158 \pm 13$ & $95 \pm 1$ & $155 \pm 20$ & $100 \pm 4$ & $317 \pm 156$ & $81 \pm 4$ & $120 \pm 19$ & $413 \pm 27$ \\
&  & Run   & $97 \pm 5$ & $64 \pm 3$ & $91 \pm 5$ & $72 \pm 1$ & $126 \pm 26$ & $51 \pm 13$ & $74 \pm 14$ & $174 \pm 18$ \\
&  & Flip  & $116 \pm 8$ & $105 \pm 4$ & $103 \pm 10$ & $87 \pm 5$ & $158 \pm 25$ & $66 \pm 8$ & $133 \pm 12$ & $299 \pm 39$ \\
\cline{3-11}
&  & Average & $222$ & $160$ & $217$ & $175$ & \underline{$303$} & $126$ & $189$ & $\textbf{394}$ \\
\cline{2-11}
& \multirow{5}{*}{Quadruped}
& Stand & $426 \pm 18$ & $72 \pm 5$ & $310 \pm 80$ & $533 \pm 37$ & $444 \pm 16$ & $421 \pm 16$ & $415 \pm 44$ & $471 \pm 27$ \\
&  & Walk  & $206 \pm 14$ & $38 \pm 5$ & $154 \pm 38$ & $254 \pm 11$ & $176 \pm 25$ & $239 \pm 12$ & $197 \pm 18$ & $257 \pm 14$ \\
&  & Run   & $195 \pm 12$ & $43 \pm 2$ & $165 \pm 57$ & $243 \pm 27$ & $172 \pm 27$ & $208 \pm 14$ & $193 \pm 19$ & $260 \pm 35$ \\
&  & Jump  & $268 \pm 17$ & $57 \pm 5$ & $210 \pm 26$ & $323 \pm 20$ & $221 \pm 23$ & $294 \pm 9$ & $256 \pm 23$ & $389 \pm 61$ \\
\cline{3-11}
&  & Average & $274$ & $53$ & $210$ & \underline{$338$} & $253$ & $318$ & $265$ & $\textbf{344}$ \\
\cline{2-11}
& \multirow{5}{*}{Cheetah}
& Walk & $625 \pm 44$ & $52 \pm 7$ & $570$ & $226 \pm 9$ & $613 \pm 77$ & $12 \pm 2$ & $273 \pm 383$ & $821 \pm 122$ \\
&  & Walk Backward & $288 \pm 126$ & $63 \pm 7$ & $37$ & $155 \pm 39$ & $371 \pm 307$ & $48 \pm 7$ & $967 \pm 8$ & $742 \pm 51$ \\
&  & Run   & $179 \pm 31$ & $30 \pm 5$ & $170$ & $68 \pm 3$ & $189 \pm 41$ & $13 \pm 5$ & $118 \pm 113$ & $275 \pm 11$ \\
&  & Run Backward & $75 \pm 7$ & $14 \pm 1$ & $6$ & $29 \pm 4$ & $59 \pm 73$ & $9 \pm 4$ & $248 \pm 46$ & $191 \pm 61$ \\
\cline{3-11}
&  & Average & $292$ & $40$ & $196$ & $120$ & $308$ & $21$ & \underline{$402$} & $\textbf{507}$ \\
\cline{2-11}
& \multirow{5}{*}{Jaco}
& Reach Top Left & $35 \pm 6$ & $2 \pm 1$ & $9 \pm 1$ & $6 \pm 2$ & $21 \pm 8$ & $22 \pm 6$ & $22 \pm 9$ & $28 \pm 12$ \\
&  & Reach Top Right & $28 \pm 15$ & $4 \pm 2$ & $29 \pm 5$ & $11 \pm 1$ & $36 \pm 13$ & $23 \pm 4$ & $14 \pm 5$ & $21 \pm 5$ \\
&  & Reach Bottom Left & $27 \pm 4$ & $2 \pm 1$ & $20 \pm 4$ & $22 \pm 5$ & $12 \pm 2$ & $36 \pm 13$ & $40 \pm 7$ & $28 \pm 2$ \\
&  & Reach Bottom Right & $12 \pm 5$ & $2 \pm 1$ & $38 \pm 1$ & $7 \pm 2$ & $29 \pm 18$ & $19 \pm 6$ & $63 \pm 11$ & $32 \pm 11$ \\
\cline{3-11}
&  & Average & $26$ & $3$ & $24$ & $12$ & $25$ & $25$ & $\textbf{35}$ & \underline{$27$} \\
\midrule
\multirow{20}{*}{APT}
& \multirow{5}{*}{Walker} 
& Stand & $718 \pm 78$ & $182 \pm 14$ & $169 \pm 8$ & $174 \pm 6$ & $768 \pm 110$ & $256 \pm 48$ & $517 \pm 16$ & $818 \pm 37$ \\
&  & Walk  & $470 \pm 20$ & $33 \pm 1$ & $66 \pm 3$ & $32 \pm 1$ & $504 \pm 79$ & $178 \pm 37$ & $477 \pm 21$ & $505 \pm 91$ \\
&  & Run   & $165 \pm 15$ & $33 \pm 3$ & $37 \pm 3$ & $28 \pm 2$ & $239 \pm 34$ & $67 \pm 6$ & $191 \pm 14$ & $220 \pm 29$ \\
&  & Flip  & $359 \pm 17$ & $44 \pm 4$ & $71 \pm 4$ & $31 \pm 2$ & $353 \pm 37$ & $77 \pm 8$ & $311 \pm 11$ & $381 \pm 55$ \\
\cline{3-11}
&  & Average & $428$ & $73$ & $86$ & $66$ & \underline{$466$} & $145$ & $374$ & $\textbf{481}$ \\
\cline{2-11}
& \multirow{5}{*}{Quadruped}
& Stand & $392 \pm 34$ & $121 \pm 5$ & $167 \pm 5$ & $389 \pm 35$ & $375 \pm 112$ & $207 \pm 25$ & $420 \pm 37$ & $425 \pm 91$ \\
&  & Walk  & $196 \pm 19$ & $51 \pm 5$ & $99 \pm 22$ & $214 \pm 7$ & $173 \pm 36$ & $109 \pm 9$ & $199 \pm 12$ & $241 \pm 49$ \\
&  & Run   & $195 \pm 35$ & $57 \pm 4$ & $74 \pm 3$ & $204 \pm 9$ & $188 \pm 38$ & $103 \pm 8$ & $201 \pm 14$ & $207 \pm 10$ \\
&  & Jump  & $293 \pm 47$ & $93 \pm 2$ & $124 \pm 8$ & $322 \pm 22$ & $270 \pm 40$ & $171 \pm 12$ & $364 \pm 25$ & $333 \pm 91$ \\
\cline{3-11}
&  & Average & $269$ & $81$ & $116$ & $282$ & $252$ & $148$ & \underline{$296$} & $\textbf{302}$ \\
\cline{2-11}
& \multirow{5}{*}{Cheetah}
& Walk & $524 \pm 42$ & $86 \pm 3$ & $114 \pm 15$ & $395 \pm 84$ & $737 \pm 54$ & $234 \pm 63$ & $899 \pm 55$ & $916 \pm 44$ \\
&  & Walk Backward & $147 \pm 81$ & $321 \pm 36$ & $320 \pm 44$ & $346 \pm 40$ & $578 \pm 137$ & $519 \pm 79$ & $604 \pm 41$ & $878 \pm 121$ \\
&  & Run   & $249 \pm 62$ & $30 \pm 1$ & $13 \pm 6$ & $69 \pm 16$ & $251 \pm 18$ & $84 \pm 17$ & $319 \pm 32$ & $306 \pm 11$ \\
&  & Run Backward & $99 \pm 60$ & $63 \pm 2$ & $61 \pm 17$ & $76 \pm 9$ & $169 \pm 109$ & $174 \pm 42$ & $144 \pm 23$ & $273 \pm 8$ \\
\cline{3-11}
&  & Average & $255$ & $125$ & $127$ & $222$ & $434$ & $253$ & \underline{$492$} & $\textbf{593}$ \\
\cline{2-11}
& \multirow{5}{*}{Jaco}
& Reach Top Left & $32 \pm 15$ & $3 \pm 1$ & $18 \pm 1$ & $38 \pm 10$ & $42 \pm 8$ & $65 \pm 19$ & $24 \pm 10$ & $69 \pm 13$ \\
&  & Reach Top Right & $20 \pm 3$ & $1 \pm 1$ & $41 \pm 4$ & $39 \pm 11$ & $64 \pm 27$ & $23 \pm 6$ & $44 \pm 8$ & $51 \pm 10$ \\
&  & Reach Bottom Left & $12 \pm 4$ & $3 \pm 1$ & $18 \pm 1$ & $9 \pm 1$ & $17 \pm 9$ & $39 \pm 6$ & $25 \pm 5$ & $24 \pm 3$ \\
&  & Reach Bottom Right & $14 \pm 4$ & $1 \pm 1$ & $25 \pm 7$ & $21 \pm 7$ & $31 \pm 32$ & $33 \pm 6$ & $27 \pm 4$ & $28 \pm 3$ \\
\cline{3-11}
&  & Average & $20$ & $2$ & $26$ & $27$ & $39$ & \underline{$40$} & $30$ & {$\textbf{43}$} \\
\bottomrule
\end{tabular}}
\label{table_zeroshot_new_all}
\vspace{-0.2cm}
\end{table*}

\clearpage
\section{Pseudo-code}
We present the pseudocode for training SRCP in the offline visual unsupervised reinforcement learning setting in Algorithm \ref{algo1}. Specifically, the procedure illustrates how SRCP incorporates the HILP method to learn skill-conditioned representations during the pretraining phase.
\begin{algorithm}[H]
\caption{SRCP Algorithm}
\begin{algorithmic}[1]
    \STATE \textbf{Inputs:} pre-collected dataset $\mathcal{D}$, randomly initialized representation network $f_\theta$, basic feature networks $\varphi_\nu$, successor feature network $\psi_\kappa$, actor network $\pi_\zeta$ , learning rate $\eta$, mini-batch size $b$, number of gradient updates step $M$, number of gradient updates every step $N$, skill update period $T$, temperature $\tau$.
    \FOR{$m = 1$ to $M$}
    \FOR{$n = 1$ to $N$}
        \STATE Sample a mini-batch of transitions $\{(o_i, a_i, o_{i+1})\}_{i \in I} \subset \mathcal{D}$ of size $|I|=b$. 
        \STATE Sample a mini-batch of target state-action pairs $\{(o'_i, a'_i)\}_{i \in I} \subset \mathcal{D}$ of size $|I|=b$.
        \STATE Sample a mini-batch of $\{z_i\}_{i \in I} \sim p(z)$ of size $|I|=b$.
        \STATE \textbf{/* Generate saliency maps */}
        \STATE Generate saliency maps $o_\alpha$ from observation $o$.
        \STATE \textbf{/* Upadate encoder */}
        \STATE Compute saliency dynamics representation learning objective $L(\theta)$ with equation (5).
        \STATE Update $\theta \gets \theta - \eta \nabla L(\theta)$.
        \STATE \textbf{/* Upadate basic feature and successor feature */}
        \STATE Compute HILP basic feature objective $L(\nu)$ with: \\
        $\mathcal{L}_{\nu} = \sum_{i=1}^{2} \mathbb{E}_{(f(o), a, f(o'), g)} \left[\left( r + \gamma (1 - r) \cdot V'_i(f(o'), g) - V_i(f(o), g) \right)^2 \right]$
        \STATE Compute successor measure loss with: \\
        $\mathcal{L}_{\kappa} = ||\varphi(f(o'))z+\gamma \psi(f(o'),a',z)z - \psi(f(o),a,z)z||^2$
        \STATE Update $\nu, \kappa \gets \nu, \kappa - \eta \nabla L(\nu) - \eta \nabla L(\kappa)$.
        \STATE \textbf{/* Upadate policy */}
        \STATE Compute actor loss function $L(\zeta)$ with equation (11).
        \STATE Update $\zeta \gets \zeta - \eta \nabla L(\zeta)$.
        \ENDFOR
        
        \STATE \textbf{/* Update target network parameters */}
        \STATE $\kappa^- \gets \tau \kappa + (1-\tau)\kappa^-$
    \ENDFOR
\end{algorithmic}
\label{algo1}
\end{algorithm}

\clearpage
\section{Theoretical analysis}
\begin{proposition}[{\cite[Prop.~16]{fb}}]
\label{prop1}
Let \( \tau: S \times A \to G \) be a mapping into a goal space \(G\).
Let \( \pi \) be a policy, and let \( M^\pi \) denote its successor state measure in goal space \(G\), as defined in (10).  
Assume that \( M^\pi(s,a,\cdot) \) is absolutely continuous with respect to a positive measure \( \rho \) on \(G\), and let \( m^\pi \) be its Radon--Nikodym derivative.  
For any measurable function \( r : G \to \mathbb{R} \), define the reward
\(R(s,a) := r(\tau(s,a)).\)
Then the \(Q\)-function of policy \( \pi \) under reward \(R\) satisfies
\begin{equation}
\begin{split}
Q^\pi(s,a)
&= \int_G r(g)\, M^\pi(s,a,\mathrm{d}g) \\
&= \int_G r(g)\, m^\pi(s,a,g)\, \rho(\mathrm{d}g).
\end{split}
\end{equation}
\end{proposition}

\begin{proof}
For each time \( t \ge 0 \), let \( P_t^\pi(s_0,a_0,\mathrm{d}g) \) denote the distribution of 
\(
g = \tau(s_t,a_t)
\)
under trajectories generated by policy \( \pi \) starting from \( (s_0,a_0) \).
By definition of the successor state measure:
\begin{equation}
M^\pi(s,a,\mathrm{d}g)
= \sum_{t\ge0} \gamma^t\, P_t^\pi(s,a,\mathrm{d}g).
\end{equation}
By definition of the density \( m^\pi \), the \(Q\)-function of \( \pi \) under reward \(R\) satisfies:
\begin{equation}
\begin{split}
Q^\pi(s,a)
&= \sum_{t\ge0} \gamma^t \mathbb{E}[R(s_t,a_t)\mid s_0=s,a_0=a,\pi]  \\
&= \sum_{t\ge0} \gamma^t \mathbb{E}[r(\tau(s_t,a_t))\mid s_0=s,a_0=a,\pi] \\
&= \sum_{t\ge0} \gamma^t \int_G r(g)\, P_t^\pi(s,a,\mathrm{d}g) \\
&= \int_G r(g)\, M^\pi(s,a,\mathrm{d}g).
\end{split}
\end{equation}
\end{proof}

\begin{proposition}[{\cite[Prop.~18]{fb}}]
\label{prop2}
Let \( f: S \times A \to \mathbb{R} \) be an arbitrary function, and define the policy \( \pi_f \) by \( \pi_f(s) := \arg\max_a f(s, a) \). 
Let \( r: S \times A \to \mathbb{R} \) be a bounded reward function, \( Q^\star \) the optimal \( Q \)-function for \( r \), and \( Q^{\pi_f} \) the \( Q \)-function of policy \( \pi_f \) under reward \( r \). Then:
\begin{equation}
    \sup_{S \times A} |f - Q^\star| \leq \frac{2}{1 - \gamma} \sup_{S \times A} |f - Q^{\pi_f}|,
\end{equation}
and
\begin{equation}
\sup_{S \times A} |Q^{\pi_f} - Q^\star| \leq \frac{3}{1 - \gamma} \sup_{S \times A} |f - Q^{\pi_f}|. \label{eq:prop18-2}
\end{equation}
\end{proposition}

\begin{proof}
We restate the proof here for completeness.  
Define the approximation error
\[
\varepsilon(s, a) := Q^{\pi_f}(s, a) - f(s, a).
\]
The \( Q \)-function \( Q^{\pi_f} \) satisfies the Bellman equation
\[
Q^{\pi_f}(s, a) = r(s, a) + \gamma \mathbb{E}_{s' | (s, a)} Q^{\pi_f}(s', \pi_f(s')).
\]
Substituting \( Q^{\pi_f} = f + \varepsilon \), we obtain
\begin{equation}
\begin{split}
f(s, a) &= r(s, a) - \varepsilon(s, a) 
+ \gamma \mathbb{E}_{s'|(s, a)} \big[f(s', \pi_f(s')) + \varepsilon(s', \pi_f(s'))\big] \\
&= r(s, a) - \varepsilon'(s, a) + \gamma \mathbb{E}_{s'|(s, a)} f(s', \pi_f(s')) \\
&= r(s, a) - \varepsilon'(s, a) + \gamma \mathbb{E}_{s'|(s, a)} \max_{a'} f(s', a'),
\end{split}
\label{eq:prop18-bellman}
\end{equation}
where 
\[
\varepsilon'(s, a) := \varepsilon(s, a) - \gamma \mathbb{E}_{s'|(s, a)} \varepsilon(s', \pi_f(s')).
\]
Equation~\eqref{eq:prop18-bellman} corresponds to the optimal Bellman equation for the modified reward function \( r - \varepsilon' \).  
Hence, \( f \) is the optimal \( Q \)-function associated with reward \( r - \varepsilon' \).  
Since \( Q^\star \) is the optimal \( Q \)-function for the original reward \( r \), it follows that
\[
\sup_{S \times A} |f - Q^\star| \leq \frac{1}{1 - \gamma} \sup_{S \times A} |\varepsilon'|.
\]
By the definition of \( \varepsilon' \), we have 
\[
\sup_{S \times A} |\varepsilon'| \leq 2 \sup_{S \times A} |\varepsilon| = 2 \sup_{S \times A} |f - Q^{\pi_f}|,
\]
which establishes the first inequality.  
The second inequality follows from the triangle inequality
\[
|Q^{\pi_f} - Q^\star| \leq |Q^{\pi_f} - f| + |f - Q^\star|,
\]
and observing that \( \frac{2}{1 - \gamma} + 1 \leq \frac{3}{1 - \gamma} \).
\end{proof}

\begin{theorem}
\label{thm:SF_error_bound}
Let $\varphi: S \times A \to \mathbb{R}^d$ be a set of base features, and let 
\(
\psi^{\pi_z}(s,a)=\int \varphi(s',a')\, M^{\pi_z}(s,a,\mathrm d s' \mathrm d a')
\)
be the true successor features of policy $\pi_z$.
Let $\hat M^{z}$ be an estimated successor measure and define the estimated successor features
\[
\hat\psi^{z}(s,a)=\int \varphi(s',a')\,\hat M^{z}(s,a,\mathrm d s' \mathrm d a').
\]
For each $z\in Z$, define the policy
\[
\pi_z(s)=\arg\max_a \hat\psi^{z}(s,a)^\top z.
\]
Assume the reward is linearly represented by the features,
\[
r(s,a)=\varphi(s,a)^\top z_r,
\]
where 
\(
z_r = \mathbb E_\rho[\varphi\varphi^\top]^{-1}\mathbb E_\rho[\varphi r]
\)
is the least–squares coefficient under a reference distribution $\rho$.
Let $V^\star$ be the optimal value function for $r$, and let $\hat V^{\pi_z}$ be the value function of $\pi_z$ under the same reward.
Then the suboptimality of the zero-shot policy $\pi_{z_r}$ is controlled by the successor-feature approximation error:
\begin{equation}
\label{eq:SF_final_bound}
\|\hat V^{\pi_{z_r}} - V^\star\|_\infty
\;\le\;
\frac{3\|z_r\|_*}{1-\gamma}\;
\sup_{s,a}\|\hat\psi^{z_r}(s,a)-\psi^{\pi_{z_r}}(s,a)\|.
\end{equation}
Moreover, for any $z\in Z$, the approximation error of $\hat\psi$ controls the deviation from the optimal $Q$-function:
\begin{equation}
\label{eq:Q_final_bound}
\sup_{s,a}\Big|\hat\psi^{z}(s,a)^\top z - Q^\star(s,a)\Big|
\;\le\;
\frac{2\|z\|_*}{1-\gamma}\;
\sup_{s,a}\|\hat\psi^{z}(s,a)-\psi^{\pi_z}(s,a)\|.
\end{equation}
\textbf{Remark.}
The bounds \eqref{eq:SF_final_bound}–\eqref{eq:Q_final_bound} show that the suboptimality of policies derived from approximate successor features is controlled by the approximation error of $\psi$, providing a quantitative bound on the zero-shot generalization performance.

\end{theorem}

\begin{proof}
Let $z \in Z$ and consider the reward $r(s,a) = \varphi(s,a)^\top z_r$.
By definition of the successor feature,
\[
\psi^{\pi_z}(s,a) = \int \varphi(s',a')\, M^{\pi_z}(s,a, \mathrm{d}s'\, \mathrm{d}a').
\]
Similarly, the estimated successor feature satisfies
\[
\hat\psi^{z}(s,a) = \int \varphi(s',a')\, \hat M^{z}(s,a, \mathrm{d}s'\, \mathrm{d}a').
\]
Define the error measure between the true and estimated successor measures:
\[
\varepsilon_z(s,a, \mathrm{d}s'\, \mathrm{d}a') := M^{\pi_z}(s,a, \mathrm{d}s'\, \mathrm{d}a') - \hat M^z(s,a, \mathrm{d}s'\, \mathrm{d}a').
\]
By Proposition~\ref{prop1} with \(G = S \times A\), the $Q$-function of policy $\pi_z$ satisfies
\begin{align}
Q^{\pi_z}(s,a) 
&= \int r(s',a')\, M^{\pi_z}(s,a, \mathrm{d}s'\, \mathrm{d}a') \notag \\
&= \int r(s',a')\, \hat M^z(s,a, \mathrm{d}s'\, \mathrm{d}a') 
   + \int r(s',a')\, \varepsilon_z(s,a, \mathrm{d}s'\, \mathrm{d}a') \notag \\
&= \hat\psi^z(s,a)^\top z_r + \int r(s',a')\, \varepsilon_z(s,a, \mathrm{d}s'\, \mathrm{d}a').
\end{align}
By the dual norm property,
\[
\Big|\int r(s',a')\, \varepsilon_z(s,a, \mathrm{d}s'\, \mathrm{d}a')\Big| 
\le \|z_r\|_* \, \|\hat\psi^z(s,a) - \psi^{\pi_z}(s,a)\|,
\]
where $\|\cdot\|_*$ is the dual norm corresponding to the norm used in the supremum over measures.
Let $f(s,a) := \hat\psi^{z_r}(s,a)^\top z_r$ and recall that $\pi_{z_r}(s) = \arg\max_a f(s,a)$.
Then, by the Proposition~\ref{prop2},
\begin{align}
\|\hat V^{\pi_{z_r}} - V^\star\|_\infty 
&\le \frac{3}{1-\gamma} \sup_{s,a} |f(s,a) - Q^{\pi_{z_r}}(s,a)| \notag \\
&\le \frac{3\|z_r\|_*}{1-\gamma} \sup_{s,a} \|\hat\psi^{z_r}(s,a) - \psi^{\pi_{z_r}}(s,a)\|.
\end{align}
Similarly, for any $z \in Z$,
\begin{align}
\sup_{s,a} \big| \hat\psi^z(s,a)^\top z - Q^\star(s,a) \big|
&\le \frac{2}{1-\gamma} \sup_{s,a} | \hat\psi^z(s,a)^\top z - Q^{\pi_z}(s,a)| \notag \\
&\le \frac{2\|z\|_*}{1-\gamma} \sup_{s,a} \|\hat\psi^z(s,a) - \psi^{\pi_z}(s,a)\|.
\end{align}
This completes the proof, establishing that the suboptimality of policies derived from approximate successor features is controlled by the approximation error of $\hat\psi$.
\end{proof}

\clearpage
\section{Extended Background}
Reinforcement Learning (RL)~\cite{sutton_rl,wang2024dynamic,zhang2022trajgen} has achieved remarkable success in learning task-specific behaviors when guided by handcraft extrinsic rewards. 
However, such supervision often requires extensive manual effort and limits generalization to novel tasks. 
Inspired by the effectiveness of unsupervised pretraining in computer vision and natural language processing, Unsupervised Reinforcement Learning (URL) has emerged as a promising paradigm that enables agents to acquire reusable representations~\cite{apt, proto} and behaviors from reward-free interactions~\cite{rnd, disagreement}.
In URL, agents are trained using task-agnostic objectives, such as intrinsic rewards or self-supervised losses, to explore environments and acquire general-purpose skills or latent representations. 
These pretrained components can later be adapted to downstream tasks with improved sample efficiency.
Among various URL approaches, unsupervised skill discovery (USD) methods~\cite{cic, diayn, metra, cesd} aim to developing diverse skills via reward-free learning to enable task generalization through fine-tuning. 
Despite these advancements, most existing USD methods focus primarily on discover diverse skills, while lacking the ability to directly generalize to new tasks without additional fine-tuning, limiting the zero-shot generalization in URL. 

A prominent line of work in zero-shot URL builds upon the framework of successor representations (SR)~\cite{sr}, which decouples environment dynamics from reward specification to support generalization across tasks. 
Extensions such as successor features (SF)~\cite{sf} further enhance this formulation by allowing skill-conditioned value estimation.
State-of-the-art zero-shot methods instantiate this idea via universal successor features (USFs)~\cite{usf} or forward-backward (FB) representations~\cite{fb, fb1}. USF-based methods typically rely on learned basic features, and a variety of techniques have been proposed to facilitate this, including Laplacian eigenfunctions~\cite{lap}, low-rank transition approximations~\cite{lra-p}, contrastive learning~\cite{cl}, and hilbert space representation learning~\cite{hilp}. In contrast, FB methods avoid explicit feature learning by directly modeling successor measures through jointly learned forward and backward representations.
While existing SR methods exhibit strong zero-shot generalization in structured state-based environments, they struggle to scale to visual tasks due to the challenges of learning meaningful representations from high-dimensional inputs.
To address this limitation, we propose Saliency-guided Representation with Consistency Policy learning (SRCP), a novel framework designed to improve generalization in visual unsupervised reinforcement learning.
SRCP enhances SR-based learning through two key innovations: (1) a saliency-guided dynamics objective that learns dynamics-relevant representations by decoupling them from SR training, and (2) a consistency-based diffusion actor that models diverse skill-conditioned behaviors with improved action controllability.

Diffusion models~\cite{diffusion} have shown impressive success in modeling complex, high-dimensional distributions and generating diverse, high-quality samples across a variety of domains, including image generation, planning~\cite{cm}, and control. 
These properties make them especially appealing for RL, where the ability to capture expressive and multi-modal action distributions is crucial for balancing exploration and exploitation~\cite{imit_diff, offline_diff, edp}. Recent works have demonstrated that diffusion-based policies can outperform traditional MLP-based actors, particularly in tasks requiring diverse and multi-modal behaviors.
However, the adoption of diffusion models in RL comes with a significant computational cost.
Their iterative sampling process leads to high inference latency, making them inefficient for real-time decision-making and online learning scenarios.
To mitigate this, recent advances have introduced consistency models~\cite{cm2024, cccp, ligeneralizing} as efficient one-step approximations of diffusion processes. These models retain the generative expressiveness of diffusion while enabling faster inference, making them promising candidates for use in RL actor networks. Empirical results have shown that consistency models can match or even surpass diffusion policies in continuous control tasks, with substantially lower computational overhead.
Despite their promising properties, consistency models remain underexplored in visual URL~\cite{urlb}, which demands learning multi-modal skill-conditioned behaviors from high-dimensional visual inputs.
To bridge this gap, we propose a consistency policy equipped with URL-specific classifier-free guidance, which promotes multi-modal action distribution modeling and skill controllability in visual tasks and improves generalization in visual URL.

\clearpage
\section{Experimental Setting}
\paragraph{Environments}
We adopt the URL Benchmark~\cite{urlb} and ExORL datasets~\cite{exorl} to evaluate the task generalization performance in visual URL. The benchmark comprises 16 high-dimensional continuous control tasks across 4 distinct domains: \textit{Walker}, \textit{Quadruped}, \textit{Cheetah}, and \textit{Jaco}.
All environments provide image-based observations with a resolution of $64 \times 64 \times 3$. Agents should extract effective representations and learn skill-conditioned behaviors from high-dimensional visual observations, making this benchmark particularly challenging for evaluating zero-shot task generalization in visual URL.
We detail the four domains as follows:
\begin{itemize}
    \item \textbf{Walker Domain} involves humanoid locomotion with  action space $\mathcal{A} \in \mathbb{R}^{6}$. It includes the tasks of \textit{Walker Stand}, \textit{Walker Walk}, \textit{Walker Flip}, and \textit{Walker Run}.

    \item \textbf{Quadruped Domain} features four-legged locomotion with higher action space $\mathcal{A} \in \mathbb{R}^{16}$. It includes the tasks of \textit{Quadruped Stand}, \textit{Quadruped Walk}, \textit{Quadruped Jump}, and \textit{Quadruped Run}.

    \item \textbf{Cheetah Domain} involves fast planar locomotion with action space $\mathcal{A} \in \mathbb{R}^{6}$. It includes the tasks of \textit{Cheetah Walk}, \textit{Cheetah Walk Backward}, \textit{Cheetah Run}, and \textit{Cheetah Run Backward}.

    \item \textbf{Jaco Domain} includes manipulation tasks using a 9-DoF Jaco robotic arm. It includes the tasks of \textit{Jaco Reach Top Left}, \textit{Jaco Reach Top Right}, \textit{Jaco Reach Bottom Left}, and \textit{Jaco Reach Bottom Right}.
\end{itemize}

\paragraph{Datasets}
To promote generalization and support reward-free pretraining, the datasets in each domain are collected using 4 distinct unsupervised exploration strategies: \textit{Random Network Distillation (RND)}, \textit{Active Pretraining with Successor Features (APS)}, \textit{Active Pretraining (APT)}, and \textit{Prototype-based Exploration (PROTO)}. Each method is designed to encourage diverse trajectories without relying on task-specific reward signals.
\begin{itemize}
    \item \textbf{Random Network Distillation (RND)}~\cite{rnd} leverages the prediction error of a randomly initialized neural network as an intrinsic reward. This encourages agents to explore novel states with high prediction error, thereby promoting broad state coverage.

    \item \textbf{Active Pretraining with Successor Features (APS)}~\cite{aps} guides exploration by maximizing the norm of successor features computed from learned representations. This drives the agent toward informative and discriminative states under the successor feature formulation.

    \item \textbf{Active Pretraining (APT)}~\cite{apt} maximizes a non-parametric entropy in an abstract representation space to drive exploration, avoiding explicit density modeling. 

    \item \textbf{Prototype-based Exploration (PROTO)}~\cite{proto} learns a set of prototypical representations that summarize the agent’s exploratory experience. These prototypes serve as a compact basis for state encoding and are learned in a self-supervised manner without downstream task labels, facilitating efficient exploration.
\end{itemize}
These strategies generate diverse and task-agnostic trajectories, which serve as a foundation for evaluating zero-shot generalization in visual URL.

\clearpage
\section{Baseline Methods}
To comprehensively evaluate the generalization of SRCP framework, we compare it against a wide range of successor representation (SR) approaches. These methods can be broadly categorized into two groups:

\paragraph{(1) Successor Feature (SF) Methods with Different Basic Feature Learning Techniques~\cite{sr}.}
These approaches follow the classical successor feature formulation but differ in how they learn or define the underlying basic representations:

\begin{itemize}
    \item \textbf{Auto encoder (AE)}~\cite{autoencoder}:
This method learns state embeddings by reconstructing raw observations through a bottleneck neural network. It encodes each input into a compact latent representation and then decodes it back to the original observation space. By minimizing reconstruction error, the learned embeddings capture salient visual features of the environment, compressing redundant information while preserving the dominant factors of variation.

    \item \textbf{Successor Features with Laplacian Eigenfunctions (Lap)}~\cite{lap}:
This method learns state embeddings by computing the eigenfunctions of the graph Laplacian constructed from state transitions. It minimizes differences between adjacent state representations under a behavior policy while ensuring orthonormality, promoting clustering among nearby states and separation of distant ones.

\item \textbf{Contrastive Learning (CL)}~\cite{cl}:
CL applies a SimCLR-style contrastive objective to learn representations by distinguishing positive state pairs (e.g., temporally close) from negatives (e.g., randomly sampled). It effectively learns features aligned with the spectral structure of the successor measure but requires full trajectories and sufficient data diversity.

\item \textbf{Low-Rank Approximation of Successor Representations (LRA-SR)}~\cite{lap}:
LRA-SR extends CL by using temporal-difference learning to estimate the successor measure, improving stability compared to Monte Carlo estimation. It normalizes feature vectors to enforce orthogonality, resulting in a low-rank approximation of the successor representation matrix.

\item \textbf{Forward Dynamics Model (FDM)}~\cite{fb1}:
FDM leverages a learned forward model to predict future state embeddings from current ones, implicitly encouraging features to capture dynamics. It is a simple and scalable way to encode temporal structure for successor learning.

\item \textbf{Hilbert Representations (HILP)}~\cite{hilp}:
HILP learns basic representations by optimizing a goal-conditioned value function through inner products in a Hilbert space. Unlike contrastive or reconstruction-based methods, HILP directly aligns feature learning with the goal-reaching structure of the MDP, enabling more effective downstream adaptation.

\end{itemize}

\paragraph{(2) Forward-Backward Representation (FB)~\cite{fb1}:}
The FB method avoids explicit feature learning by directly estimating successor measures using forward and backward representations. It circumvents the challenges of learning stable embeddings and has achieved state-of-the-art results in zero-shot URL. 

These baselines cover a wide spectrum of successor representation learning strategies for URL. SRCP is evaluated in comparison to each of these to highlight its generalization capabilities in visual URL. In particular, SRCP differs by incorporating saliency-guided dynamics representation learning and consistency policy modeling, allowing it to achieve superior zero-shot performance on high-dimensional visual tasks.

\clearpage
\section{Hyper-parameter settings}

In this section, we provide the detailed hyper-parameter settings of our proposed SRCP, as shown in Table \ref{hyper}.

\begin{table}[h]
\caption{Hyper-parameter settings.}
\centering
\renewcommand\arraystretch{1.1}
\setlength{\tabcolsep}{3mm}{
\begin{tabular}{lc}
\hline
Hyper-parameter                      & Setting       \\ \hline

Pre-training frames & $5e^5$ \\
Zero-shot selection frames & $1e^4$ \\
RL replay buffer size & $1e6$ \\
Frame stack & $3$\\
Action repeat & $1$\\
$z$ vector dimensions             & $50$ \\
$z$ vector space            &    continuous   \\
RL backbone algorithm & DDPG \\
Return discount & $0.99$\\
Batch size & $512$ \\
Optimizer & Adam \\
Learning rate & $1e-4$\\
Actor activation  & ${\rm layernorm(Tanh)} \rightarrow {\rm ReLU}  $  \\
Agent update frequency & $2$ \\
Target critic network EMA & $0.01$\\
Exploration stddev clip & $0.3$ \\
Exploration stddev value & $0.2$ \\
coefficient $\omega$  & $3$ \\
coefficient $\beta$  & $0.5$ \\
coefficient $\lambda_1$  & $0.2$ \\
coefficient $\lambda_1$  & $0.2$ \\
\hline
\end{tabular}
}
\label{hyper}
\end{table}

\clearpage
\section{Understanding the Limits of Visual Representation Learning for Successor Representations}

\paragraph{Can We Directly Learn Physical State Representations from Pixels?}
To assess whether physical state supervision can effectively guide representation learning in visual URL, we explore the feasibility of regressing low-dimensional physical state vectors directly from image observations. 
Specifically, we decouple the training into two stages: (1) an encoder is trained via supervised regression to predict physical states from visual inputs, and (2) the resulting features are used to train a successor measure module for downstream tasks.
As shown in Fig.\ref{figa2}(a), the representation learned directly from ground-truth physical states achieves convergence in state prediction loss, but the converged value remains high. 
Furthermore, Fig.\ref{figa2}(b) shows that, compared to learning representations directly from visual inputs using the SR objective, using ground-truth physical states does not lead to significant performance improvement.
This suggests that, even with access to accurate physical states, the encoder still struggles to recover precise state dynamics.
This limitation carries over to the successor feature learning phase, where no notable improvement in generalization is observed. 
These findings suggest that accurately reconstructing physical states from high-dimensional images is inherently challenging, and that such reconstruction-based objectives may be insufficient for effective successor training and task generalization.

\begin{figure}[h]
\centering
\includegraphics[width=1.0\columnwidth]{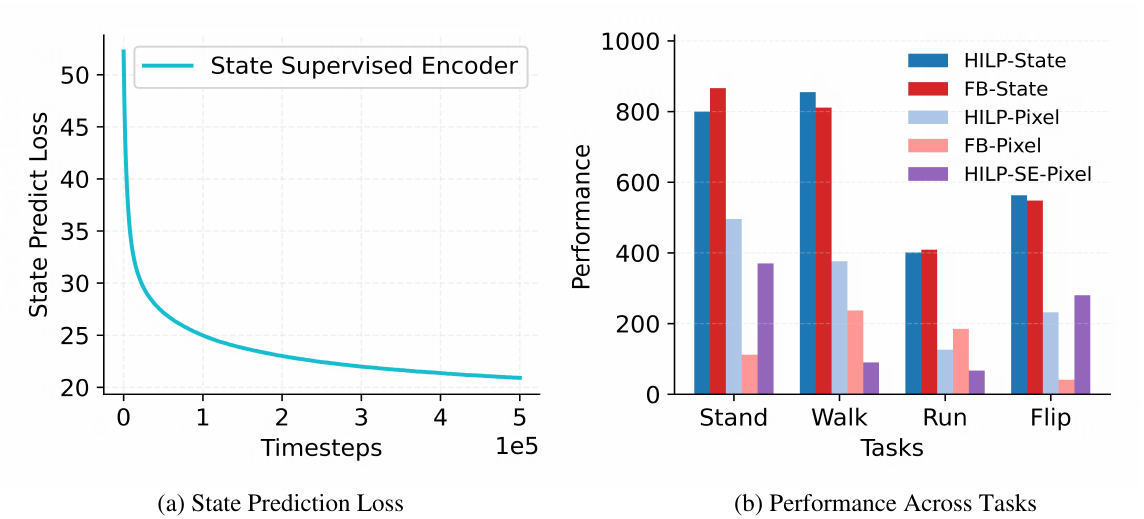} 
\caption{Experimental results of training an encoder using physical states as supervision in the Walker domain. (a) Prediction loss of the encoder during training; (b) Comparison of HILP and FB across different tasks under state input, visual input, and HILP with state-supervised encoder.}
\label{figa2}
\end{figure}

\paragraph{Beyond Physical States: The Role of Dynamics-Relevant Attention.}
Recognizing the limitations of physical state regression, we pose a more fundamental question: what kind of representation is truly useful for successor representation (SR) learning in visual URL?
To explore this, we conduct attention analysis comparing the focus of encoders trained with conventional SR objectives versus those guided by our proposed saliency-based approach. As illustrated in the main text, conventional SR methods tend to capture static or background regions in observations, which offer poor dynamics-relevant representations. In contrast, our Saliency-Guided Dynamics Encoder (SDE) consistently attends to dynamics-relevant regions that are crucial for accurate successor estimation.
This reveals a key insight: rather than attempting to reconstruct physical states, it is more effective to learn dynamics-relevant representations that suitable for successor measure. By introducing the saliency-guided dynamic representation learning, our method explicitly encourages the encoder to focus on dyanmics-relevant regions, resulting in significantly improved generalization in visual URL.

\clearpage
\section{Effectiveness of Saliency-Guided Dynamic Representation Learning}
To evaluate the impact of Saliency-Guided Dynamic Representation Learning on learned representations, we provide detailed comparisons of saliency maps and activation heatmaps produced by our encoder and those from existing SR methods. As shown in Fig.\ref{figak}, baseline methods such as FB and HILP tend to focus on task-irrelevant regions of the observation, whereas our SRCP framework consistently attends to more accurate dynamics-relevant regions. This demonstrates the effectiveness of our method in guiding the encoder to learn more meaningful dynamics-relevant representations.
\begin{figure}[h]
\centering
\includegraphics[width=0.7\columnwidth]{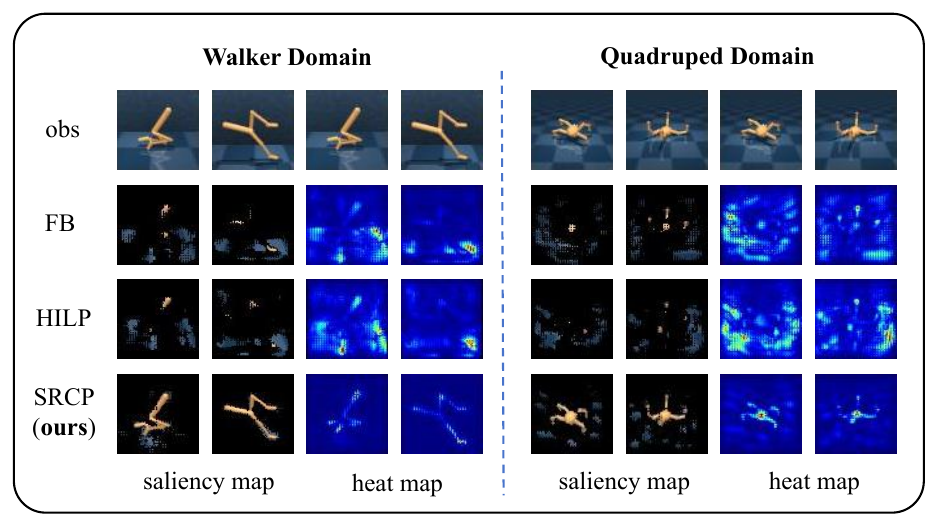} 
\caption{Detailed attention heat map of SR methods.}
\label{figak}
\end{figure}

To further evaluate the quality of representations learned by SRCP’s saliency-guided dynamics encoder, we assess its ability to capture essential physical properties of the environment. Specifically, we compare SRCP with two SR method baselines, HILP and FB, by examining how effectively their pretrained encoders support the prediction of physical state information.
Following a linear probing protocol, we freeze the encoder parameters and append a randomly initialized linear layer. 
This linear layer is trained to regress the underlying physical state vectors from the latent representations in the \textit{Walker} domain. 
We track both the convergence speed and final regression accuracy to measure the informativeness of the learned features.
As shown in Fig.\ref{figa4}, the SRCP encoder achieves significantly faster convergence and lower prediction error than those from HILP and FB. 
This indicates that the saliency-guided dynamic representation learning in SRCP leads to more physically effective representations, which are better aligned with the true underlying dynamics of the environment.
Such representations are essential not only for accurate successor measure estimation but also for zero-shot task generalization in visual URL.

\begin{figure}[h]
\centering
\includegraphics[width=0.45\columnwidth]{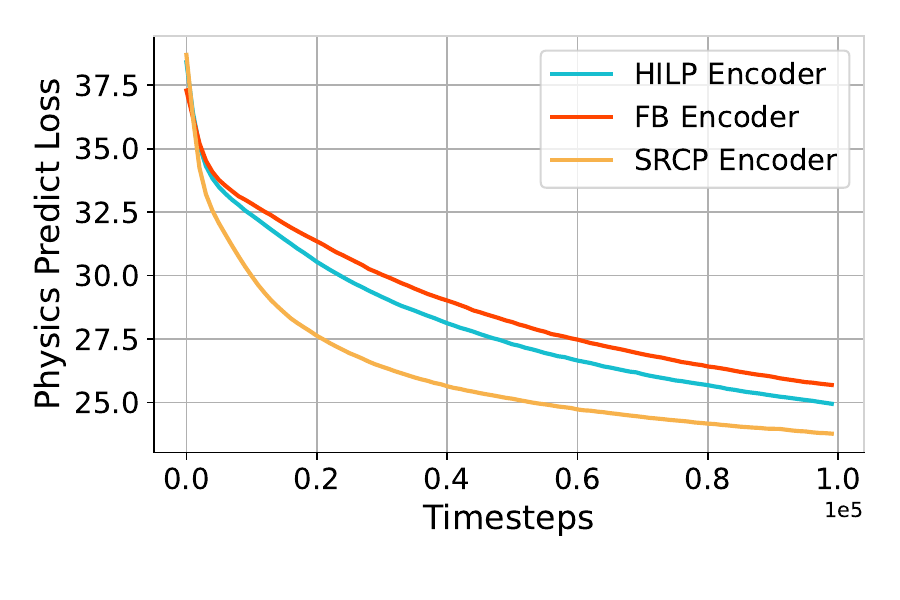} 
\caption{Prediction loss of physical states in walker domain.}
\label{figa4}
\end{figure}

\clearpage
\section{Full Ablation Study in RND Dataset}
To investigate the contribution of each component in the SRCP framework, we conduct an ablation study on all 16 tasks from the RND dataset using 4 random seeds. The compared variants include:
\begin{itemize}
    \item \textbf{HILP}: A baseline that jointly trains the encoder and successor features using the Hilbert basic feture learning and SR objective.
    \item \textbf{SRCP w/o SE}: A variant of SRCP that removes the saliency-guided dynamics encoder, using the consistency policy with an encoder trained via Hilbert basic feature learning and the SR objective.
    \item \textbf{SRCP w/o CM}: A variant of SRCP that removes the consistency actor, employing the saliency-guided dynamics encoder while using standard policy learning.
    \item \textbf{SRCP}: The complete SRCP method, combining saliency-guided dynamics representation learning with consistency-based policy learning.
\end{itemize}
As shown in Table~\ref{table_zeroshot_ablation_all}, both \textbf{SRCP w/o SE} and \textbf{SRCP w/o CM} outperform the HILP baseline across most domains, confirming the benefits of both the saliency-guided dynamics encoder and the consistency policy. Specifically, \textbf{SRCP w/o CM} performs better than HILP in most tasks, demonstrating that saliency-guided representations improve generalization. 
Meanwhile, \textbf{SRCP w/o SE} achieves higher performance in several tasks, indicating the effectiveness of the consistency policy in modeling multi-modal skill-conditioned behaviors. 
The full \textbf{SRCP} model consistently achieves the best overall results across all domains, validating that both components work synergistically to enhance generalization in visual URL.

\begin{table}[h]
\caption{Full Ablation Study}
\setlength{\tabcolsep}{5pt}
\renewcommand\arraystretch{1.15}
\centering
\resizebox{0.6\textwidth}{!}{
\begin{tabular}{cccccc}
\toprule
Domain & Task & HILP & SRCP w/o SE & SRCP w/o CM & SRCP \\
\midrule
\multirow{5}{*}{Walker} 
& Stand & $496 \pm 73$ & $519 \pm 12$ & $658 \pm 38$ & $671 \pm 51$ \\
& Walk & $376 \pm 52$ & $423 \pm 12$ & $449 \pm 75$ & $520 \pm 32$ \\
& Run & $126 \pm 8$ & $145 \pm 4$ & $173 \pm 21$ & $194 \pm 38$ \\
& Flip & $232 \pm 41$ & $293 \pm 5$ & $303 \pm 13$ & $369 \pm 25$ \\
\cline{2-6}
& Average & $308$ & $345$ & $396$ & $439$ \\
\midrule
\multirow{5}{*}{Quadruped}
& Stand & $327 \pm 126$ & $476 \pm 9$ & $595 \pm 30$ & $703 \pm 22$ \\
& Walk & $163 \pm 45$ & $253 \pm 15$ & $305 \pm 14$ & $339 \pm 27$ \\
& Run & $148 \pm 19$ & $289 \pm 35$ & $310 \pm 32$ & $361 \pm 24$ \\
& Jump & $244 \pm 122$ & $391 \pm 34$ & $413 \pm 18$ & $535 \pm 14$ \\
\cline{2-6}
& Average & $221$ & $352$ & $406$ & $485$ \\
\midrule
\multirow{5}{*}{Cheetah}
& Walk & $895 \pm 33$ & $893 \pm 44$ & $807 \pm 26$ & $859 \pm 34$ \\
& Walk Backward & $927 \pm 35$ & $951 \pm 23$ & $954 \pm 22$ & $934 \pm 45$ \\
& Run & $276 \pm 46$ & $277 \pm 14$ & $292 \pm 25$ & $322 \pm 44$ \\
& Run Backward & $297 \pm 46$ & $278 \pm 3$ & $338 \pm 6$ & $293 \pm 42$ \\
\cline{2-6}
& Average & $599$ & $600$ & $598$ & $602$ \\
\midrule
\multirow{5}{*}{Jaco}
& Top Left & $29 \pm 4$ & $30 \pm 8$ & $49 \pm 3$ & $53 \pm 9$ \\
& Top Right & $35 \pm 13$ & $54 \pm 2$ & $63 \pm 9$ & $60 \pm 7$ \\
& Bottom Left & $25 \pm 3$ & $43 \pm 3$ & $23 \pm 4$ & $43 \pm 8$ \\
& Bottom Right & $31 \pm 11$ & $49 \pm 7$ & $39 \pm 4$ & $42 \pm 8$ \\
\cline{2-6}
& Average & $30$ & $44$ & $43$ & $50$ \\
\bottomrule
\end{tabular}}
\label{table_zeroshot_ablation_all}
\vspace{-0.2cm}
\end{table}

\clearpage
\section{Effectiveness of Representation Learning Methods for Visual URL}

To evaluate the impact of different representation learning strategies under the HILP framework, we compare five variants on the RND dataset: the original HILP, HILP-TACO, HILP-TACOP, HILP-SE (SRCP w/o CM), and the full SRCP method. HILP-TACO and HILP-TACOP incorporate the state-of-the-art representation learning methods TACO~\cite{taco} and TACO-Premier~\cite{tacop} into HILP, aiming to decouple the encoder from the successor representation objective. In contrast, HILP-SE and SRCP adopt our proposed saliency-guided dynamics encoder to learn dynamics-relevant visual representations.

As shown in Table~\ref{table:hilp_variant_enc}, HILP-TACO and HILP-TACOP achieve limited improvements over HILP, with performance gains in some tasks (e.g., walker\_stand or jaco\_reach\_top\_left) but substantial drops in others (e.g., walker\_walk or cheetah\_run\_backward), leading to inconsistent generalization. 
This suggests that applying generic representation learning alone is insufficient for achieving robust generalization performance across diverse tasks.
In contrast, HILP-SE significantly outperforms all three baselines in every domain, particularly in \textit{Walker} and \textit{Quadruped}, where it improves average zero-shot returns by over 90\% compared to HILP. Furthermore, SRCP consistently achieves the best performance across most tasks, demonstrating that combining the saliency-guided dynamics encoder with consistency policy learning further enhances multi-modal action distribution modeling ability and generalization.
These results clearly highlight the superiority of the proposed saliency-guided dynamics representation in visual URL. 
It enables the agent to focus on dynamics-relevant features and supports more effective zero-shot generalization.

\begin{table}[h]
\caption{Zero-shot Performance of HILP Variants on the RND Dataset}
\setlength{\tabcolsep}{5pt}
\renewcommand\arraystretch{1.15}
\centering
\resizebox{0.7\textwidth}{!}{
\begin{tabular}{ccccccc}
\toprule
Domain & Task & HILP & HILP-TACO & HILP-TACOP & HILP-SE & SRCP \\
\midrule
\multirow{5}{*}{Walker}
& Stand & $496 \pm 73$ & $388 \pm 8$ & $622 \pm 24$ & $658 \pm 38$ & $671 \pm 51$ \\
& Walk & $376 \pm 52$ & $111 \pm 6$ & $339 \pm 14$ & $449 \pm 75$ & $520 \pm 32$ \\
& Run & $126 \pm 8$ & $73 \pm 3$ & $135 \pm 10$ & $173 \pm 21$ & $194 \pm 38$ \\
& Flip & $232 \pm 41$ & $99 \pm 8$ & $199 \pm 24$ & $303 \pm 13$ & $369 \pm 25$ \\
\cline{2-7}
& Average & $308$ & $168$ & $324$ & $396$ & $439$ \\
\midrule
\multirow{5}{*}{Quadruped}
& Stand & $327 \pm 126$ & $234 \pm 18$ & $195 \pm 16$ & $595 \pm 30$ & $703 \pm 22$ \\
& Walk & $163 \pm 45$ & $91 \pm 19$ & $122 \pm 13$ & $305 \pm 14$ & $339 \pm 27$ \\
& Run & $148 \pm 19$ & $73 \pm 10$ & $101 \pm 6$ & $310 \pm 32$ & $361 \pm 24$ \\
& Jump & $244 \pm 122$ & $86 \pm 13$ & $149 \pm 16$ & $413 \pm 18$ & $535 \pm 14$ \\
\cline{2-7}
& Average & $221$ & $121$ & $142$ & $406$ & $485$ \\
\midrule
\multirow{5}{*}{Cheetah}
& Walk & $895 \pm 33$ & $693 \pm 33$ & $599 \pm 33$ & $807 \pm 26$ & $859 \pm 34$ \\
& Walk Backward & $927 \pm 35$ & $422 \pm 15$ & $534 \pm 89$ & $954 \pm 22$ & $934 \pm 45$ \\
& Run & $276 \pm 46$ & $136 \pm 6$ & $185 \pm 7$ & $292 \pm 25$ & $322 \pm 44$ \\
& Run Backward & $297 \pm 46$ & $80 \pm 2$ & $124 \pm 24$ & $338 \pm 6$ & $293 \pm 42$ \\
\cline{2-7}
& Average & $599$ & $333$ & $361$ & $598$ & $602$ \\
\midrule
\multirow{5}{*}{Jaco}
& Top Left & $29 \pm 4$ & $46 \pm 4$ & $33 \pm 2$ & $49 \pm 3$ & $53 \pm 9$ \\
& Top Right & $35 \pm 13$ & $43 \pm 3$ & $24 \pm 3$ & $63 \pm 9$ & $60 \pm 7$ \\
& Bottom Left & $25 \pm 3$ & $25 \pm 6$ & $41 \pm 3$ & $23 \pm 4$ & $43 \pm 8$ \\
& Bottom Right & $31 \pm 11$ & $23 \pm 4$ & $24 \pm 6$ & $39 \pm 4$ & $35 \pm 10$ \\
\cline{2-7}
& Average & $30$ & $34$ & $31$ & $43$ & $48$ \\
\bottomrule
\end{tabular}}
\label{table:hilp_variant_enc}
\vspace{-0.2cm}
\end{table}

\clearpage
\section{Experimental Result on Forward Backward Representation}

To further validate the generality of SRCP, we integrate the Forward-Backward (FB)~\cite{fb} representation into the SRCP framework, denoted as SRCP(FB). 
Table~\ref{table:hilp_variant_rnd} presents the detailed zero-shot performance of FB and SRCP(FB) across four domains (Walker, Quadruped, Cheetah, and Jaco) in the APS dataset. 
SRCP(FB) consistently achieves higher returns in all tasks, with particularly notable improvements in the Walker and Cheetah domains. 
These results align with the findings in the main paper, confirming that SRCP provides a general and effective framework for enhancing the representation quality, skill expressiveness, and zero-shot generalization of various SR-based methods.

\begin{table}[h]
\caption{Zero-shot Performance of FB and SRCP(FB) on the APS Dataset}
\setlength{\tabcolsep}{5pt}
\renewcommand\arraystretch{1.15}
\centering
\resizebox{0.4\textwidth}{!}{
\begin{tabular}{cccc}
\toprule
Domain & Task & FB & SRCP(FB) \\
\midrule
\multirow{5}{*}{Walker}
& Stand & $304 \pm 122$ & $680 \pm 14$ \\
& Walk & $81 \pm 14$ & $478 \pm 18$ \\
& Run & $51 \pm 13$ & $152 \pm 12$ \\
& Flip & $66 \pm 8$ & $274 \pm 13$ \\
\cline{2-4}
& Average & $126$ & $\textbf{396}$ \\
\midrule
\multirow{5}{*}{Quadruped}
& Stand & $421 \pm 16$ & $583 \pm 21$ \\
& Walk & $239 \pm 12$ & $294 \pm 10$  \\
& Run & $208 \pm 14$ & $293 \pm 8$ \\
& Jump & $294 \pm 9$ & $421 \pm 32$ \\
\cline{2-4}
& Average & $318$ & $\textbf{398}$ \\
\midrule
\multirow{5}{*}{Cheetah}
& Walk & $12 \pm 2$ & $291 \pm 84$ \\
& Walk Backward & $48 \pm 7$ & $495 \pm 19$ \\
& Run & $13 \pm 5$ & $57 \pm 7$ \\
& Run Backward & $9 \pm 4$ & $103 \pm 13$ \\
\cline{2-4}
& Average & $21$ & $\textbf{237}$ \\
\midrule
\multirow{5}{*}{Jaco}
& Top Left & $22 \pm 6$ & $46 \pm 6$ \\
& Top Right & $23 \pm 4$ & $21 \pm 3$ \\
& Bottom Left & $36 \pm 13$ & $48 \pm 15$ \\
& Bottom Right & $19 \pm 6$ & $21 \pm 5$ \\
\cline{2-4}
& Average & $25$ & $\textbf{34}$\\
\bottomrule
\end{tabular}}
\label{table:hilp_variant_rnd}
\vspace{-0.2cm}
\end{table}

\clearpage
\section{Full Hyperparameter Ablation Study}
We provide a comprehensive ablation study on the sensitivity of SRCP to two key hyperparameters, $\omega$ and $\beta$.
Table~\ref{table_zeroshot_ablation3} shows effect of $\omega$ in Walker domain, $\omega$ controls the strength of skill-conditioned guidance. Setting $\omega=0$ leads to a marked performance drop, while moderate values enhance generalization. In contrast, large $\omega$ over-constrains the policy and slightly reduces performance.
Table~\ref{table_zeroshot_ablation2n} analyzes the effect of $\beta$ in Walker domain, which controls the weight of saliency guidance in representation learning. Increasing $\beta$ improves generalization up to a point, after which excessive emphasis on salient regions weakens the model’s ability to capture dynamics cues.
Overall, SRCP maintains stable performance across a wide range of hyperparameters, with balanced $\omega$ and $\beta$ achieving the best trade-off between skill and saliency guidance.

\begin{table}[h]
\setlength{\tabcolsep}{6pt}
\renewcommand\arraystretch{1.2}
\centering
\resizebox{0.5\textwidth}{!}{
\begin{tabular}{cccccc}
\toprule
Walker & $\omega=0$ & $\omega=2$ & $\omega=3$ & $\omega=4$& $\omega=6$ \\
\midrule
Stand    & $176 \pm 7$  & $609 \pm 40$ & $\textbf{671} \pm \textbf{51}$ & $498 \pm 30$ & $363 \pm 16$ \\
Walk     & $42 \pm 4$  & $415 \pm 40$ & $\textbf{520} \pm \textbf{32}$ & $464 \pm 16$ & $79 \pm 13$ \\
Run      & $32 \pm 6$  & $136 \pm 18$ & $\textbf{194} \pm \textbf{38}$ & $164 \pm 11$ & $70 \pm 14$ \\
Flip     & $59 \pm 7$  & $\textbf{374} \pm \textbf{18}$ & $369 \pm 25$ & $358 \pm 18$ & $206 \pm 10$ \\
\cline{1-6}
Average     & $77$  & $384$ & $\textbf{439}$ & $371$ & $180$ \\
\bottomrule
\end{tabular}}
\caption{Ablation study of parameter $\omega$ in SRCP on Walker domain in RND dataset, with 4 random seeds per task.}
\label{table_zeroshot_ablation3}
\end{table}

\begin{table}[h]
\setlength{\tabcolsep}{6pt}
\renewcommand\arraystretch{1.2}
\centering
\resizebox{0.45\textwidth}{!}{
\begin{tabular}{ccccc}
\toprule
Walker & $\beta=0$ & $\beta=0.2$ & $\beta=0.5$ & $\beta=1$ \\
\midrule
Stand    & $544 \pm 20$  & $550 \pm 26$ & $\textbf{671} \pm \textbf{51}$ & $583 \pm 18$  \\
Walk     & $354 \pm 42$  & $458 \pm 19$ & $\textbf{520} \pm \textbf{32}$ & $484 \pm 48$ \\
Run      & $156 \pm 12$  & $170 \pm 35$ & $\textbf{194} \pm \textbf{38}$ & $169 \pm 23$ \\
Flip     & $310 \pm 15$  & $315 \pm 11$ & $\textbf{369} \pm \textbf{25}$ & $346 \pm 5$ \\
\cline{1-5}
Average     & $341$  & $373$ & $\textbf{439}$ & $396$ \\
\bottomrule
\end{tabular}}
\caption{Ablation study of parameter $\beta$ in SRCP on Walker domain in RND dataset, with 4 random seeds per task.}
\label{table_zeroshot_ablation2n}
\vspace{-0.5cm}
\end{table}
\clearpage
\section{Computational Cost}
As shown in Table~\ref{tab:training_cost}, the proposed SRCP framework achieves training time and memory consumption that are comparable to those of the baseline methods across all experimental settings.
By explicitly decoupling representation learning and value learning into separate optimization objectives, SRCP avoids the computationally expensive joint optimization commonly adopted in conventional end-to-end frameworks.
This structural design effectively offsets the additional computational overhead introduced by saliency map generation, leading to similar wall-clock training time compared to the baselines.
Overall, SRCP only incurs minor and acceptable extra overhead.
On the one hand, the decoupled training paradigm eliminates chained backpropagation across multiple modules, which significantly reduces gradient computation complexity.
On the other hand, the saliency inference module is designed to be lightweight and efficient, introducing negligible latency during training.
Meanwhile, the consistency-based policy regularization imposes computational costs comparable to those of standard policy optimization objectives, without introducing heavy additional computation.
Consequently, the total training cost of SRCP remains competitive with state-of-the-art approaches despite the integration of multiple learning objectives.

\begin{table}[h!]
\centering
\caption{Computational cost on an A800 GPU for 5M steps.}
\label{tab:training_cost}
\begin{tabular}{lccc}
\toprule
Metric & FB & HILP & SRCP \\
\midrule
Training Time (h)$\downarrow$ & \textbf{23.1} & 24.2 & 24.4 \\
Memory Usage (GB)$\downarrow$ & 9.6 & \textbf{8.8} & 9.1 \\
\bottomrule
\end{tabular}
\end{table}

\section{Visually Complex Environments}
To further validate the generalization ability of SRCP in complex and noisy visual environments, we evaluate its performance on the DMC-GB benchmark, which contains various types of visual distractions.
Specifically, we train SRCP on environments with clean backgrounds and then test it under two challenging visual settings for the Walker task: Color Easy (mild color perturbations) and Color Hard (severe color variations).
As reported in Table~\ref{tab:dmcgb_color}, SRCP consistently outperforms the baseline methods FB and HILP across both settings, which demonstrates its strong generalization capability against diverse task dynamics and visual appearance variations.
In addition to visual robustness, we assess the scalability and cross-domain applicability of SRCP beyond locomotion tasks.
To this end, we conduct experiments on the Point Mass Maze navigation task, a representative goal-directed control problem distinct from continuous locomotion.
The quantitative results are presented in Table~\ref{tab:point_mass_maze}.
SRCP achieves significantly improved performance compared with all baselines, indicating that the proposed framework can effectively transfer its learning paradigm to different task domains and maintain strong decision-making performance.
For a more comprehensive evaluation under complex visual perturbations, we further test SRCP on the full DMC-GB benchmark with multiple types of visual disturbances.
The agent is trained on clean backgrounds and then evaluated under three increasingly challenging test scenarios: dynamic easy color perturbations, hard color perturbations, and natural video distractions.
Consistent with previous observations, the results in Table~\ref{tab:dmcgb_color} confirm that SRCP still outperforms FB and HILP by a clear margin, further verifying its superior generalization ability in the presence of complex and realistic visual variations.
In summary, experiments on both the DMC-GB visual generalization benchmark and the Point Mass Maze navigation task demonstrate that SRCP not only achieves strong robustness to visual noise and domain shifts but also scales effectively across different task categories, showing wide applicability in visual reinforcement learning.

\begin{table}[h!]
\centering
\caption{Performance on DMC-GB Walker.}
\label{tab:dmcgb_color}
\resizebox{.8\columnwidth}{!}{%
\begin{tabular}{l l ccc ccc ccc}
\toprule
 &  & 
\multicolumn{3}{c}{Color Easy} 
& \multicolumn{3}{c}{Color Hard} 
& \multicolumn{3}{c}{Video Easy} \\
\cmidrule(lr){3-5} \cmidrule(lr){6-8} \cmidrule(lr){9-11}
Domain & Task 
& FB & HILP & SRCP
& FB & HILP & SRCP
& FB & HILP & SRCP \\
\midrule
\multirow{5}{*}{Walker}
& Stand & $243 \pm 80$ & $364 \pm 89$ & $639 \pm 69$ & $165 \pm 66$ & $282 \pm 74$ & $506 \pm 58$ & $127 \pm 42$ & $176 \pm 32$ & $277 \pm 48$ \\
& Walk  & $62 \pm 29$  & $110 \pm 54$ & $423 \pm 47$ & $54 \pm 27$  & $73 \pm 34$  & $279 \pm 45$ & $47 \pm 14$  & $62 \pm 13$  & $162 \pm 47$ \\
& Run   & $46 \pm 25$  & $55 \pm 18$  & $165 \pm 31$ & $39 \pm 16$  & $49 \pm 21$  & $113 \pm 36$ & $38 \pm 8$   & $40 \pm 5$   & $78 \pm 10$ \\
& Flip  & $59 \pm 22$  & $98 \pm 34$  & $371 \pm 42$ & $51 \pm 25$  & $69 \pm 36$  & $156 \pm 38$ & $47 \pm 11$  & $56 \pm 11$  & $123 \pm 44$ \\
\cmidrule(lr){2-5} \cmidrule(lr){6-8} \cmidrule(lr){9-11}
& Average 
&  103 & 157 & \textbf{400}
&  77 &  118 &  \textbf{264} 
&  65 &  84 &  \textbf{160} \\
\bottomrule
\end{tabular}}
\end{table}

\begin{table}[h!]
\centering
\caption{Performance on Point Mass Maze navigation tasks.}
\label{tab:point_mass_maze}
\resizebox{.7\columnwidth}{!}{%
\begin{tabular}{l c c c c c}
\toprule
Method & Reach Top Left & Reach Top Right & Reach Bottom Left & Reach Bottom Right & Average \\
\midrule
FB    & $14 \pm 5$   & $3 \pm 1$   & $1 \pm 1$   & $0 \pm 0$   & $5$ \\
HILP  & $471 \pm 51$ & $112 \pm 19$ & $36 \pm 10$ & $2 \pm 1$   & $155$ \\
SRCP  & $635 \pm 95$ & $166 \pm 55$ & $67 \pm 18$ & $10 \pm 6$  & $\textbf{220}$ \\
\bottomrule
\end{tabular}}
\end{table}

\section{Difference between Diffusion Models and Consistency Models}
We further compare SRCP with the representative diffusion policy on the challenging Proto Quadruped locomotion tasks to verify its efficiency and practical deployment potential.
The quantitative comparison is summarized in Table~\ref{tab:diffusion_vs_srpc}.
On one hand, the full diffusion model with 40 denoising steps achieves slightly higher task performance than SRCP, but this improvement comes at the cost of drastically increased computation: it consumes nearly 2.5× the total training time of SRCP.
On the other hand, when diffusion is constrained to only a single denoising step for inference efficiency, its performance drops significantly and becomes considerably worse than SRCP.
In contrast, SRCP maintains strong task performance with inherently efficient single-step execution, without relying on iterative denoising.
These results clearly demonstrate that the proposed consistency policy effectively balances performance and computational overhead, yielding a superior performance–computation trade-off compared with diffusion-based policies for real-world robotic control.

\begin{table}[h!]
\centering
\caption{Comparison between diffusion models and SRCP.}
\label{tab:diffusion_vs_srpc}
\resizebox{.8\columnwidth}{!}{%
\begin{tabular}{
l c c c c c c
}
\toprule
Method & Stand$\uparrow$ & Walk$\uparrow$ & Run$\uparrow$ & Jump$\uparrow$ & Average$\uparrow$ & Time (h) $\downarrow$ \\
\midrule
Diffusion (1 step)   
& $626 \pm 50$ & $308 \pm 37$ & $324 \pm 18$ & $397 \pm 40$ & 414 & \textbf{24.7} \\
Diffusion (40 steps) 
& $724 \pm 30$ & $357 \pm 43$ & $354 \pm 17$ & $531 \pm 26$ & \textbf{492} & 64.8 \\
SRCP                 
& $703 \pm 22$ & $339 \pm 27$ & $361 \pm 24$ & $535 \pm 14$ & 485 & 25.4 \\
\bottomrule
\end{tabular}}
\end{table}

\section{Inference Budget} 
We further evaluate the sensitivity of SRCP to the number of task-specific transitions used during adaptation, by conducting experiments on the Walker domain with a wide range of inference budgets (as shown in Table~\ref{tab:inference_budget}).
When only 100 task-specific transitions are available, the overall performance shows a slight decrease due to the limited task-specific information.
However, across the range from 500 up to 20,000 transitions, the performance of SRCP remains stable and consistent without significant degradation.
This observation demonstrates that SRCP can achieve reliable adaptation with only a modest number of task-specific transitions, and is robust to the choice of inference budget, making it suitable for real-world scenarios where collecting large amounts of task-specific data can be costly or impractical.

\begin{table}[h!]
\centering
\caption{Effect of inference transition on SRCP performance.}
\label{tab:inference_budget}
\resizebox{.665\columnwidth}{!}{%
\begin{tabular}{lcccccc}
\toprule
Inference Transitions: & 100 & 500 & 1k & 5k & 10k & 20k \\
\midrule
Stand & $828 \pm 26$ & $817 \pm 19$ & $823 \pm 35$ & $834 \pm 34$ & $830 \pm 39$ & $834 \pm 32$ \\
Walk  & $466 \pm 87$ & $478 \pm 35$ & $508 \pm 47$ & $504 \pm 46$ & $504 \pm 66$ & $506 \pm 31$ \\
Run   & $207 \pm 42$ & $216 \pm 33$ & $224 \pm 22$ & $220 \pm 26$ & $227 \pm 16$ & $238 \pm 12$ \\
Flip  & $376 \pm 86$ & $418 \pm 26$ & $413 \pm 24$ & $434 \pm 35$ & $428 \pm 75$ & $420 \pm 28$ \\
\midrule
\rowcolor{blue!10}
Average & 469 & 482 & 492 & 498 & 497 & \textbf{500} \\
\bottomrule
\end{tabular}}
\end{table}


\end{document}